\documentclass[12pt]{article}

\usepackage[english]{babel}

\usepackage[letterpaper,top=2cm,bottom=2cm,left=3cm,right=3cm,marginparwidth=1.75cm]{geometry}

\usepackage{amsthm}
\usepackage{amsfonts,amssymb}
\usepackage{graphicx}
\usepackage{booktabs}
\usepackage{siunitx}
\usepackage{array}
\usepackage[colorlinks=true, allcolors=blue]{hyperref} 
\usepackage{subcaption}
\usepackage{cleveref}

\Crefname{equation}{equation}{equations} 
\usepackage{hyphenat}
\usepackage{multirow}
\usepackage{array}
\usepackage{authblk}

\newcolumntype{P}[1]{>{\arraybackslash}m{#1}}
 
\usepackage{threeparttable}

\begin{document}
\title{Metrics for quantifying isotropy in high dimensional unsupervised clustering tasks in a materials context}
\renewcommand*{\Affilfont}{\small\itshape}
\author[1,2]{Samantha~Durdy$^{\ast,}$}
\author[2,3]{Michael~W.~Gaultois}
\author[1,2]{Vladimir~Gusev}\author[1]{Danushka~Bollegala} 
\author[2,3]{ Matthew~J.~Rosseinsky}
\affil[1]{Department of Computer Science, University of Liverpool, Ashton Street, Liverpool, L69~3BX, UK}
\affil[2]{Leverhulme Research Centre for Functional Materials Design, University of Liverpool, 51 Oxford Street, Liverpool, L7~3NY, UK}
\affil[3]{Department of Chemistry, University of Liverpool, Crown St, Liverpool, L69~7ZD, UK}
\affil[ ]{$\ast$~E-mail: samantha.durdy@liverpool.ac.uk}


\maketitle

\begin{abstract}
Clustering is a common task in machine learning, but clusters of unlabelled data can be hard to quantify. The application of clustering algorithms in chemistry is often dependant on material representation. Ascertaining the effects of different representations, clustering algorithms, or data transformations on the resulting clusters is difficult due to the dimensionality of these data. 

We present a thorough analysis of measures for isotropy of a cluster, including a novel implantation based on an existing derivation. Using fractional anisotropy, a common method used in medical imaging for comparison, we then expand these measures to examine the average isotropy of a set of clusters. A use case for such measures is demonstrated by quantifying the effects of kernel approximation functions on different representations of the Inorganic Crystal Structure Database. Broader applicability of these methods is demonstrated in analysing learnt embedding of the MNIST dataset. Random clusters are explored to examine the differences between isotropy measures presented, and to see how each method scales with the dimensionality. Python implementations of these measures are provided for use by the community. 
\end{abstract}
\noindent{\it machine learning, unsupervised clustering, clustering, data science, materials science \/}

\section{Introduction}
Clustering algorithms have become a vital tool in materials science for tasks such as machine learning evaluation~\cite{Durdy2022,Roter2022} and data exploration~\cite{Zhang2019}. In the field of materials science, datasets can often be high-dimensional and lack target labels, making the task of clustering data a challenging one. The appropriate representation of a material is often unclear~\cite{Durdy2022, Murdock}, and manual evaluation of identified clusters is infeasible due to the size of the datasets~\cite{ICSD,matbench,materialsProject}. 

Recent publications in materials science have introduced new clustering techniques which can work in a supervised or semi supervised manner~\cite{ILS}. While new techniques show promise, and are seeing uptake~\cite{Roter2022,ILS1,ILS2,ILS3}, like many other clustering algorithms~\cite{kmeans,DBSCAN,wardClustering}, these techniques rely on distance similarity metrics (usually Euclidean distance) in order to cluster the data. Thus, these algorithms are dependant on the materials' representation. But, it is unclear which material representation is most appropriate, thus evaluating the results of clustering is important, regardless of which clustering algorithm is used. As datasets can be too large to evaluate manually, semi-automatic, or automatic metrics must be used to quantify characteristics of clusters, and the success of clustering. This paper discusses and presents such metrics.

Existing metrics for clustering unlabelled data, such as silhouette or Folkes--Mallows scores, focus on quantifying the compactness of individual clusters, and the separation of clusters from each other~\cite{Liu2010}. One aspect of sets of clusters, which remains difficult to quantify, is their average shape. In this paper we focus on the isotropy of a cluster - do the points in a cluster form a round shape, or are they a ``spikier" shape. While other shape markers like squareness or hexagonality may be given more importance, isotropy could still be a relevant factor to some researchers. 

Isotropy of data representations has been associated with improved performance in downstream machine learning tasks~\cite{isotropy}, and anisotropic clusters of data may be indicative of representations which are dominated by specific features or correlations of features. Where representation is unclear, such as materials science, exploring the effects of representation on the isotropy of clusters can be informative to a researcher. 

The shape of a cluster can be visually observed by projection of data onto a lower dimensional space~\cite{PCA,TSNE}, but this projection excludes significant amounts of information, and observations about these projections are subjective. For example, previous work used visual inspection of Principal Component Analysis (PCA) to project representations of the Inorganic Crystal Structure Database (ICSD) into 3 dimensions and qualitatively described changes in cluster shape upon application of different kernel approximation methods~\cite{Durdy2022}. Although qualitative observation can be helpful, there is an unmet need for robust metrics to quantify the isotropy of clusters and understand the underlying structure of the data. As will be seen, such numerical analysis can highlight non-intuitive results which are present at higher dimensions.

 In this report, we expand on existing methods to quantify cluster shape by measuring the isotropy (``roundness") or anisotropy (``spikiness"). We examine the robustness of existing metrics, and expand them from use on one cluster to use on a set of clusters. Metrics of this kind are commonly used in three dimensions in the field of medical imaging to identify the diffusion of water in the brain~\cite{fractional_Anisometry_book}. In higher dimensions, similar metrics arise in the context of data science to draw conclusions about the shape of point clouds~\cite{isotropy}. We provide an analysis of existing methods in high dimensions using random matrix theory and outline an alternative implementation for an existing derivation of isotropy. We extend the concept of measuring isotropy to examine sets of clusters, and provide example uses of this extension in the material science and data science domains. 
 
After a brief introduction to metrics for unlabelled clustering and metrics for isotropy, we expand an existing derivation to define our new isotropy metric. We demonstrate the usefulness of this new isotropy metric for quantifying the clustering of the ICSD, which is one of the foundational datasets in inorganic chemistry and materials science. These metrics of isotropy are used to quantify the shape of clusters within the ICSD when using different representations, and non-linear (kernel-approximation) transformations.

While this technique was developed in a material science context, it has a broader applicability wherever the shape of sets of clusters must be quantified. For example, analysing learnt embeddings is a common task in machine learning~\cite{ElMD,mat2Vec,PCA_projected} and often relies on low dimensional visualisation methods~\cite{PCA,TSNE}. As such, we further demonstrate the usefulness of this metric by using it to quantify differences in learnt embeddings of images of digits using the Modified National Institute of Standards and Technology (MNIST) dataset, a foundational data science dataset~\cite{mnist}. 

To examine the difference between the metrics for isotropy used, clusters of random points are generated in various dimensionalities and their isotropy is measured. Using mathematical tools from the field of random matrix theory~\cite{MarcenkoPastur} we are able to explain the behaviour observed in the existing metrics. Time complexity of the proposed metrics is examined and the advantages and disadvantages of each metric are discussed.

The specific contributions of this paper are as follows:
\begin{itemize}
    \item Exploring how metrics used for measuring isotropy in 3 dimensions~\cite{Fractional_Anisoptry} generalise to higher dimensions.
    \item Providing a new implementation for an isotropy measure based on an existing mathematical derivation (\Cref{alternatives}).
    \item Proposing adaptions to the measures of isotropy for single clusters such that one can measure the average isotropy across a set of clusters (\Cref{sets_of_clusters}).
    \item Highlighting the need for analysis of representation when clustering datasets relating to materials (\Cref{materials_section}).
    \item Demonstrating analysis of isotropy in a supervised learning context using a foundational data science dataset (\Cref{MNISTsection}). 
    \item Examining the robustness of the metrics under random noise perturbations (\Cref{dummy_data_experiments}). 
    \item Using random matrix theory to prove that the measurements of isotropy are related to the dimensionality of data, especially if the data are noisy (\Cref{random_matrix_theory}).
    
\end{itemize}

\section{Metrics for unsupervised clustering}
While supervised clustering tasks allow for the use of metrics such as the adjusted mutual information score~\cite{adjusted_mutual_info}, Folkes--Mallows score~\cite{Fowlkes1983}, or homogeneity and completeness scores~\cite{homogeniety}, the selection for unsupervised clustering is more sparse. Unsupervised metrics must rely on features present in the data, thus, are dependant on data representation. Because of this reliance, such metrics can be referred to as ``internal clustering validation measures''~\cite{Liu2010}. Internal clustering validation measures aim to quantify the quality of a set of clusters in an abstract sense, by focusing on either the compactness of each cluster or the separation between clusters. To contextualise the new implementations presented here, examples of prominent unsupervised clustering metrics are outlined in this section.
\par Using distance metrics such as the Euclidean distance, one can compute the average distance between each point and every other point in its cluster. This can become computationally expensive because the calculation of the pairwise distance matrix scales with the square of the number of points (\Cref{tab:metricComparison}). Thus, the average distance between a point in a cluster and its centroid can be used instead, which scales linearly with the number of points in the cluster. These computations provide a measure of how tightly packed a cluster is in the space distance is being measured over (for example, the Euclidean space). This is useful for numerically comparing clusters and clusterings of points which exist in the same space, such as comparing clusters found using different clustering algorithms. The representation of data, and any transformations will affect the distance measurements, so, this use of distance metric based quantification of clusters is inadequate for making comparisons between sets of clusters found on different representations of data.
\par The silhouette score uses distance measurements to provide a number bounded between -1 and 1 to measure how well a point is clustered~\cite{silhouetteScore}. Where 1 is considered a well clustered point (i.e., according to this metric the point is in the correct cluster) and -1 is considered a poorly clustered point (i.e., according to this metric this point should be in a different cluster). It is calculated by comparing the mean distance between a point and other members of that point's cluster, to the mean distance of that point and all members of its next closest cluster (the cluster who's members are on average closest to that point). By calculating the mean silhouette score for all points one can obtain a score for the quality of a set clusters. This score can be compared to silhuoette scores found using alternative representations of data points or kernel transformations. 
\par The Davies--Bouldin index provides a lower bounded metric of clusters without requiring a pairwise distance between points in dataset~\cite{Davies1979}. By comparing distances of points in a cluster to its centroid and distances between cluster centroids, a score is calculated with a minimum of 0, where lower scores indicate better clusterings. 
\par The Calinski--Harabasz measure (or variance ratio criteron)~\cite{Calinski1974} is a metric considering both the dispersion and separation of clusters. It is calculated using the sum of square distances between a point in a cluster and its centroid and the sum of square distances of the cluster's centroid from the mean data point in the dataset. In other words, it compares the dispersion within clusters to the dispersion of centroids in the representation space. Unlike the Davies--Bouldin index, a higher Calinski--Harabasz measure indicates a more separated set of clusters.
\par While evaluating based on dispersion and cluster compactness provides a metric as to how ``good'' an application of clustering is, it does not provide information about the clusters themselves. It can be pertinent to the use case of clustering algorithms to have evenly sized clusters (\textit{i.e.} clusters should contain approximately the same number of data points)~\cite{LOCOCV}. In this case, the variance between cluster size has been used as a metric~\cite{Durdy2022}.
\par Another property which has been difficult to reason with about clusters is their shape. We present a novel application of the isotropy metrics in a clustering context.

\begin{table}
\begin{threeparttable}
\caption{A summary of metrics for unsupervised learning. Included are descriptions, optimal values where applicable, whether their output has upper bounds and/or lower bounds ($L.B.$), and approximations of time complexity. Where $\mathcal{D}$ is a set of clusters, $|\mathcal{C}|$ is the size of a cluster, $\mathcal{E}$ is the set of all points in all clusters of $\mathcal{D}$ (thus $|\mathcal{E}| = \Sigma_{\mathcal{C}\in \mathcal{D}} |\mathcal{C}|$), $n$ is the number of dimensions in which the cluster exists and $r$ is the number of random vectors used for $I_{rnd}$.}\label{tab:metricComparison}
\centering
\begin{tabular}{p{3.2cm}p{6cm}p{1.5cm}p{1.7cm}c}
\hline
\hline
\noalign{\smallskip}
     Metric & Description  & Optimal value & Bounded & Complexity\\\hline
     \hline
     Mean distance to centroid &  Measures compactness of clusters & Min & $L.B. = 0$ & $O(|\mathcal{E}|)$\\
     \nohyphens{Mean distance between points in cluster} & Measures compactness of clusters & Min & $L.B. = 0$ & $O\left(|\mathcal{E}||\mathcal{C}|\right) $\\
     Silhouette score & Measures how close a point is to other points in its cluster compared to points in other clusters & Max & $-1,1$ & $O\left(|\mathcal{E}|^2\right)$\\
     \nohyphens{Davies--Bouldin index} & Ratio of within cluster distances to between cluster distances   & Min & $L.B. = 0$ & $O(|\mathcal{E}|)$~\tnote{*}\\
     Calinski--Harabasz measure & Ratio of between and within cluster dispersion, weighted by the size of the cluster & Max & $L.B. = 0$ & $O(|\mathcal{E}|)$\\
     Cluster size variance & Measures how evenly sized clusters are & N/A~\tnote{\textdagger}  & $L.B. = 0$ & $O(|\mathcal{D}|)$\\
     Fractional isotropy & Measures the shape of clusters &  N/A \tnote{\textdaggerdbl} & 0, 1 & $O(|\mathcal{E}|n^2)$~\tnote{\S}\\
     \nohyphens{Isotropy (EigenVec)} & Measures the shape of clusters & N/A~\tnote{\P} & $0,1$ & $O(|\mathcal{E}|n^2)$~\tnote{\S}\\
     \nohyphens{Isotropy (random)} & Measures the shape of clusters & N/A~\tnote{\P} & $0,1$ & $O(|\mathcal{E}|r)$\\
     \hline\hline

     \end{tabular}
 \begin{tablenotes}
 \item[*] Assuming $|D|>|D|^2$ otherwise $O(|D|^2)$.
 \item[\textdagger] Lower indicates more evenly sized clusters.
 \item[\textdaggerdbl] Lower indicates more isotropic.
 \item[\P] Higher indicates more isotropic.
 \item[\S]Assuming $n < |\mathcal{C}|$ else $O(|\mathcal{E}||\mathcal{C}|n)$.
 \end{tablenotes}
 \end{threeparttable}
\end{table}

\section{Metrics for Isotropy}
\begin{figure}
    \centering
    \begin{subfigure}{0.45\linewidth}
        \includegraphics[width=\linewidth]{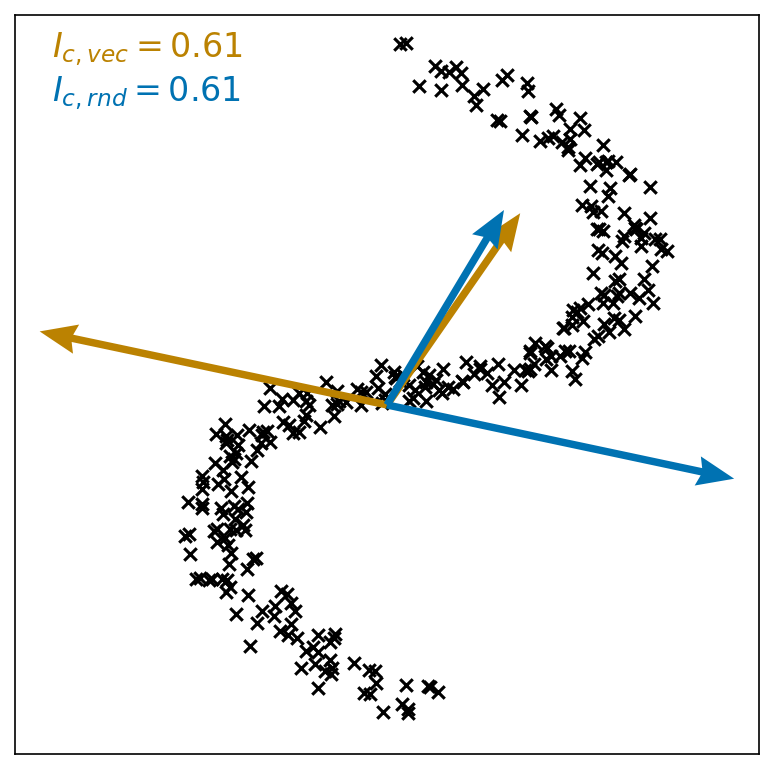}
        \caption{\centering}
        \label{fig:s_shaped_cluster}
    \end{subfigure}
    \begin{subfigure}{0.45\linewidth}
        \includegraphics[width=\linewidth]{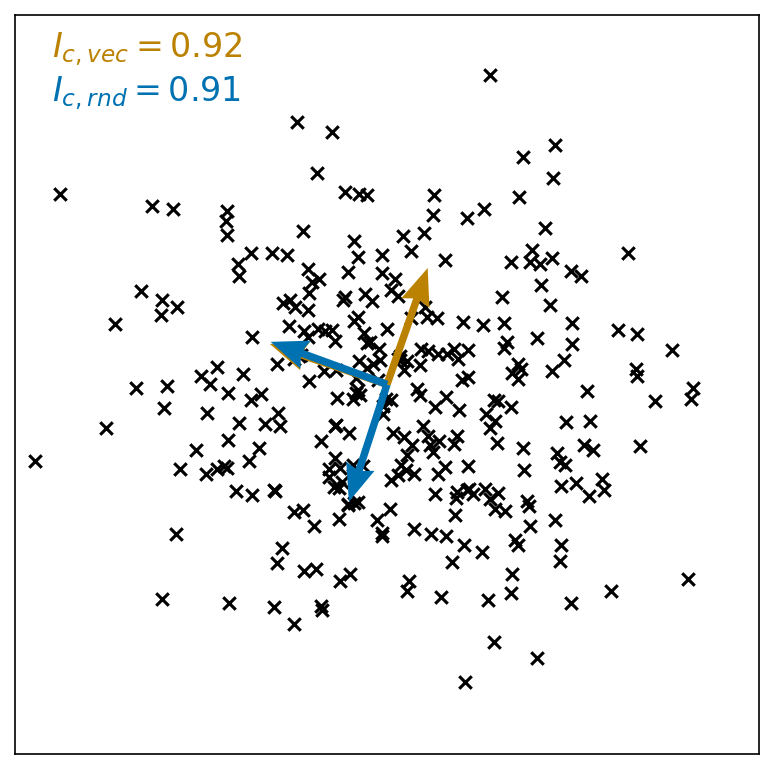}
        \caption{\centering}
        \label{fig:single_clusters}
    \end{subfigure}
    \begin{subfigure}{0.45\linewidth}
        \includegraphics[width=\linewidth]{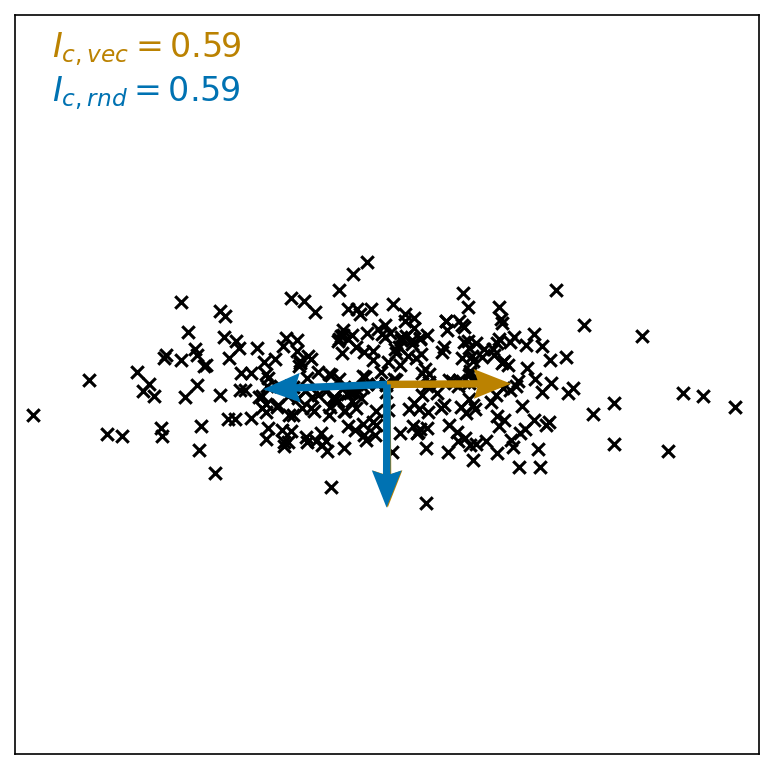}
        \caption{\centering}
        \label{fig:squashed_cluster}
    \end{subfigure}
    \begin{subfigure}{0.45\linewidth}
        \includegraphics[width=\linewidth]{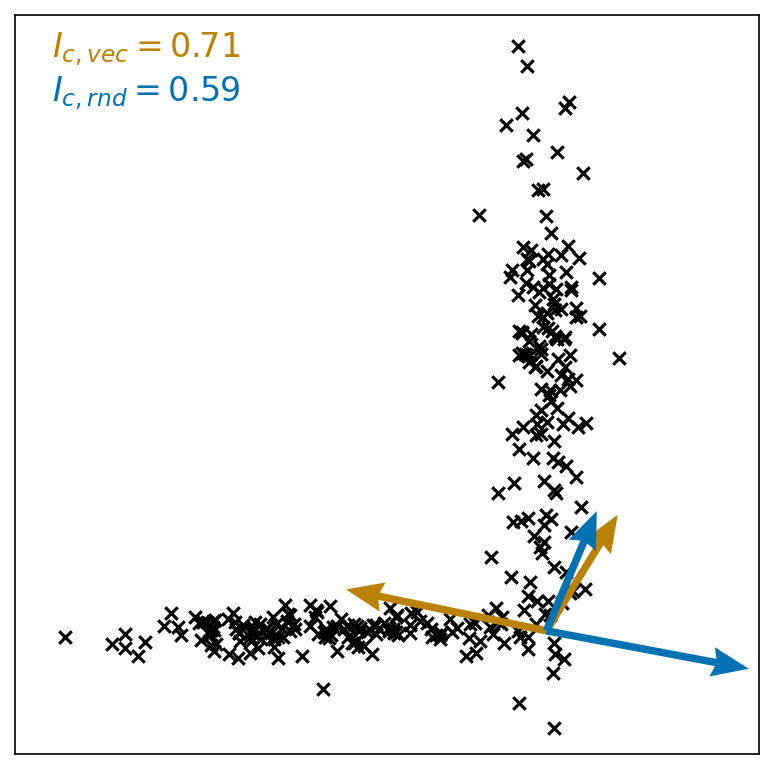}
        \caption{\centering}
        \label{fig:two_clusters}
    \end{subfigure}

    \caption{Examples of two dimensional clusters of points, labelled with measurements for $I_{c,\mathrm{vec}}$ and $I_{c,\mathrm{rnd}}$. Unit vectors $\mathbf{a}$ which resulted in $min(Z(\mathbf{a}))$ and $max(Z(\mathbf{a}))$ are shown, and colour coded according to the metric for which they have been used. Each cluster consists of 300 points. As was the case for most other low dimensional experiments (\Cref{dummy_data_experiments}), in all the examples shown here, $I_{c,\mathrm{rnd}}$ is seen to be lower than $I_{c,\mathrm{vec}}$ and thus is more accurate in these cases (Equation~\ref{eq:I_c_true_le_I_c}) (a) An s-shaped, or spiral cluster (b) A cluster picked from a two dimensional Gaussian distribution (c) A cluster picked from a two dimensional Gaussian distribution where the Y axis has a lower standard deviation than the X axis (d) A reverse L shaped cluster.}
    \label{fig:cluster_examples}
\end{figure}
As discussed above, often the shape of the clusters is an important consideration. A cluster with an isotropic shape, where the distribution of points is roughly equal in all directions, can be preferable to a cluster with a highly elongated shape~\cite{Zmeasure, isotropy}. Elongated clusters could be indicative of outliers: in clustering methods which only consider single linkage when creating clusters (\textit{e.g}, hierarchical agglomerative clustering~\cite{hierarchicalclustering} or iterative label spreading~\cite{ILS}), long chains of outliers may be grouped together~\cite{hierarchicalclustering}. In this section, we will discuss possible pitfalls when attempting to measure isotropy, and explore various metrics for measuring the isotropy of clusters, providing a quantitative way to evaluate and compare different cluster shapes.

One common pitfall when analysing the isotropy of a cluster of points is that it can be highly subjective as to whether a cluster is anisotropic or isotropic. For example, a spiral cluster (\Cref{fig:s_shaped_cluster}) may appear anisotropic if the spirals are loose. However, as spirals become closer together, or longer it becomes more subjective as to whether this cluster can be considered anisotropic or if the spirals have collapsed into a single isotropic cluster. Similarly an L-shaped cluster (\Cref{fig:two_clusters}) may seem anisotropic, but could arguably be two isotropic clusters which have been wrongly grouped.  Reducing complex correlations to single numbers will necessarily remove some information. Isotropy may be a secondary consideration compared to other markers of shape (for example, spiral, square, or hexagonal). Nevertheless, when isotropy is a primary concern for a researcher, having the ability to quantify isotropy is useful for analytical and descriptive purposes.

While metrics such as kurtosis or variance may measure the spread of a cluster in individual dimensions, anisotropy may occur between dimensions (\Cref{fig:two_clusters}) thus, a more complex statistical analysis must be used. Two important properties for a measure of isotropy to have, are \textit{(a)} being invariant against uniform scaling and \textit{(b)} being invariant against linear isometries. In other words applying linear transformations such as translations, reflections, rotations or uniform scaling to a cluster of points should not change its isotropy measurement. Confusingly, in the field of probability theory, functions which satisfy the property of being invariant under linear isometries are sometimes referred to as being ``isotropic measures'' or ``isotropic processes''~\cite{isotropic_measures}, not to be confused with the measures for isotropy or metrics for isotropy which are explored in this paper.

A simple way to incorporate invariance upon linear isometries as a metric for isotropy is to base that metric on the principal components of a cluster. Whilst the eigenvectors that make up these principal components may change with rotation, their relationship with points in the cluster and the set of eigenvalues ($\Lambda$) associated with each eigenvector will not change. Using the variance of these eigenvalues would be a simple proxy for isotropy. If the variance between eigenvalues is large, there are large differences between the eigenvalues, then there are principal components that are more significant, and thus a cluster will become anisotropic. If the variance between the eigenvalues is small, then the eigenvalues are similar and the cluster extends evenly in each of the principal axes. If the eigenvalues are normalised before measuring the variance (Var), then this proxy is invariant upon uniform scaling and invariant upon linear isometries. These normalised eigenvalues are usually denoted as the set $\lambda$:
\begin{eqnarray}\label{eq:normalised_lambda}
\lambda_i = \frac{\Lambda_i}{\sum_{j=0}^n \Lambda_j}    
\end{eqnarray}

Where $\Lambda$ are the eigenvalues of the principle components of a cluster. As uniform scaling of points in the cluster will equally scale $\sum_{j=0}^n \Lambda_j$, Var($\lambda$) is a simple measure for isotropy which is invariant upon linear isometries and upon uniform scaling.

While this means that $\mathrm{Var}(\lambda)$ is theoretically bounded between 0 and 0.25 (\Cref{thm:varianceProof}), $\mathrm{Var}(\lambda)$ is usually very small and does not use the whole of this range (\Cref{tab:ICSDMeasures,tab:MNISTMeasures}). Consequently, Var$(\lambda)$ is not a common measure of isotropy and will not be focused on here.

Similar measures for isotropy based on eigenvalues of the principal components are widely used in the field of diffusion tensor imaging. In diffusion tensor imaging, isotropy can be used as a proxy for the flow of water in the brain or spinal chord~\cite{fractional_Anisometry_book}. While many metrics for isotropy have been proposed in the medical imaging field~\cite{Alexander2007}, the most widely used of these is the Fractional Anisotropy (FA)~\cite{Fractional_Anisoptry}. FA is defined as the square root of the variance of the normalised eigenvalues of a covariance matrix, divided by the expected value of the square of the normalised eigenvalues, as given by Equation \ref{eq:FA_deffinition}. Due to its application in magnetic resonance imaging, this is often defined in three dimensions:

\begin{eqnarray}
\mathrm{FA}(\lambda) &= \sqrt{\frac{3}{2}}\sqrt{\frac{\left(\lambda_1 - \hat{\lambda}\right)^2+\left(\lambda_2 - \hat{\lambda}\right)^2+\left(\lambda_3 - \hat{\lambda}\right)^2}{\lambda_1^2+\lambda_2^2+\lambda_3^2}} \\[0.15cm]\label{eq:FA_deffinition} 
&= \sqrt{\frac{\mathrm{Var}(\lambda)}{\mathrm{E}\left(\lambda^2\right)}} = \sqrt{1-\frac{E(\lambda)^2}{E(\lambda^2)}}
\end{eqnarray}

Where $\hat{\lambda}$ is the mean of $\lambda$ ($\hat{\lambda} = (\sum_{i=0}^n \lambda)/n$). FA is bounded between 0 and 1 with 1 indicating a highly anisotropic cluster and 0 indicating a highly isotropic cluster. In medical imaging $\lambda$ is usually the set of normalised eigenvalues for the principal components of a diffusion tensor. This diffusion tensor is a small voxel of a larger medical image allowing for mapping of water diffusion in parts of the brain~\cite{fractional_Anisometry_book}. We are unaware of any higher-dimensional applications of this metric or applications on larger sets of points. We investigate the use of FA for quantifying anisotropy in larger point clouds and in higher dimensions. As will be seen shortly, special considerations are needed when using FA in higher dimensions.

While FA is popular in medical imaging, an alternative approach to measuring isotropy of a cluster was proposed in the natural language processing domain~\cite{isotropy}. This research aimed to quantify changes to high-dimensional word embeddings. This research used a previously defined function to quantify the cosine similarity between a vector and a cluster of points~\cite{Zmeasure}:
\begin{eqnarray}\label{eq:Z}
    Z(\mathbf{a}) = \sum_{\mathbf{d} \in \mathcal{C}}\exp\left(\mathbf{a}^\intercal \mathbf{d}\right)
\end{eqnarray}

In an isotropic cluster, $\mathcal{C}$, of data points, $\mathbf{d}$, the value of $Z(\mathbf{a})$ should be approximately constant with any unit vector $\mathbf{a}$. The ratio between the largest and smallest values of $Z(\mathbf{a})$ for a cluster $\mathcal{C}$ can be used to define an isotropy measurement, $I_c$, with a range between 0 and 1. The true ratio of largest to smallest values of $Z(\mathbf{a})$, would be calculated using every $\mathbf{a}$ on the unit sphere~\cite{isotropy}:
\begin{eqnarray} \label{eq:variant_isotropy}
\frac{\min_{|\mathbf{a}|=1}Z(\mathbf{a})}{\max_{|\mathbf{a}|=1}Z(\mathbf{a})}
\end{eqnarray}

However, this definition is not invariant under linear isometries or uniform scaling: moving a cluster away from the origin will result in a smaller measure for isoptropy (\Cref{thm:invarianceProof}). In order to make this definition invariant upon linear isometries and uniform scaling, we adjust the $Z(\mathbf{a})$ function used in previous work:

\begin{eqnarray}\label{eq:ZPrime}
    Z'(\mathbf{a}) = \sum_{\mathbf{d} \in \mathcal{C}}\exp\left(\mathbf{a}^\intercal \left(\frac{\mathbf{d-\hat{d}}}{\mu}\right)\right)
\end{eqnarray}
$\mathbf{\hat{d}}$ is the centroid (or mean) of $\mathcal{C}$ and $\mu$ is the mean magnitude of $\mathbf{d-\hat{d}}$: 
\begin{eqnarray}\mu=\frac{\sum_{\mathbf{d}\in \mathcal{C}}{|\mathbf{d-\hat{d}}|}}{ |\mathcal{C}|}
\end{eqnarray}

This allows the adjustment definition of a measure, $I_c$, for the isotropy of a cluster, which is bounded between 0 and 1, and invariant upon linear isometries and uniform scaling:

\begin{eqnarray}\label{eq:I_c_true}
    I_{c,\mathrm{true}} = \frac{\min_{|\mathbf{a}|=1}Z'(\mathbf{a})}{\max_{|\mathbf{a}|=1}Z'(\mathbf{a})}
\end{eqnarray}

 While linear isometries applied to $\mathcal{C}$ may change $Z'(\mathbf{a})$ for a single value of $\mathbf{a}$, intuitively it will not change the value of $I_{c,\mathrm{true}}$. Note that for $I_{c,\mathrm{true}}$ (and its non-invariant counterpart in Equation~\ref{eq:variant_isotropy}) an isotropic cluster will result in a measurement close to 1 and an anisotropic cluster will be close to 0. This is the opposite to the to FA which is 1 for anisotropic clusters and 0 for isotropic clusters.
 
As the set of vectors on the unit sphere is infinite, previous work~\cite{isotropy} approximated $I_{c,\mathrm{true}}$ by measuring $Z$ for the set eigenvectors found in PCA. As we will later propose an alternative approximation (\Cref{alternatives}), for clarity we label this implementation of $I_c$ as $I_{c,\mathrm{vec}}$. $I_{c,\mathrm{vec}}$ (adjusted for invariance under scaling and linear isometries) is thus defined by:
\begin{eqnarray}\label{eq:I_cvec}
    I_{c,\mathrm{vec}}(C)\approx \frac{\min_{\mathbf{a}\in A}  Z'(\mathbf{a})}{\max_{\mathbf{a}\in A}  Z'(\mathbf{a})}
\end{eqnarray}
where $\mathbf{a}$ is the set of eigenvectors found by applying SVD to $\mathcal{C}^\intercal \mathcal{C}$. Readers familiar with PCA will note that this is the same process by which PCA is calculated (though PCA here will be applied to the cluster). Thus $\mathbf{a}$ is the set of eigenvectors that are the principal components of the cluster. 
\subsection{Alternative interpretation of isotropy definition}\label{alternatives}
Much like FA, $I_{c,\mathrm{vec}}$ assumes that isotropy originates from the principal axis of a cluster. As seen empirically, this is often a valid assumption (\Cref{tab:ICSDMeasures,tab:MNISTMeasures,fig:dimVsPerformance}), but it is possible to think of clusters for which this is clearly not the case (\Cref{fig:two_clusters}). 

We approach the task of finding the set $\mathcal{B}$ for which we measure $Z'(\mathbf{a})$ as an optimisation task. The set of points on the unit sphere is of infinite size, thus, a subset of unit vectors $\mathcal{B}$ must be defined for which to calculate $\forall_{\mathbf{a}\in \mathcal{B}}Z'(\mathbf{a})$: 
\begin{eqnarray}\label{eq:I_c_given_B}
        I_{c|\mathcal{B}}(\mathcal{C})\approx \frac{\min_{\mathbf{b}\in \mathcal{B}}  Z'(\mathbf{b})}{\max_{\mathbf{b}\in \mathcal{B}}  Z'(\mathbf{b})}
\end{eqnarray}
A good approximation of $I_{c,\mathrm{true}}$ must not incur excess computation and must be accurate. The computational complexity can be estimated theoretically and stated using the ``big O notation'' (examples in ~\Cref{tab:metricComparison}) or can be measured experimentally (\Cref{fig:timeVsDims}). Assessing which set $\mathcal{B}$ provides the most accurate $I_{c|\mathcal{B}}$ is also straightforward. Given our definition of $I_{c,\mathrm{true}}$ (\ref{eq:I_c_true}), $I_{c|\mathcal{B}}$ will always be an upper bound for the true value of $I_{c,\mathrm{true}}$ (~\Cref{thm:upperBoundProof}):
\begin{eqnarray}\label{eq:I_c_true_le_I_c}
    \forall_{c,\mathcal{B}}: I_{c,\mathrm{true}}\le I_{c|\mathcal{B}}
\end{eqnarray}
From this we can conclude that, given two sets of vectors to constitute $\mathcal{B}$, the one for which $I_{c|\mathcal{B}}$ is smaller will be the more accurate of the two. Thus, this task is framed as finding a set $\mathcal{B}$, which will result in the smallest value of $I_{c|\mathcal{B}}$, while taking into account the incurred computational costs.

As a possible solution, we propose using a random set of unit vectors, $r$, to define the set $\mathcal{B}$. We label this solution $I_{c,\mathrm{rnd}}$, which can be defined as:
\begin{eqnarray}\label{eq:I_c_rnd}
        I_{c,\mathrm{rnd}}(\mathcal{C})\approx \frac{\min_{\mathbf{b}\in r}  Z(\mathbf{b})}{\max_{\mathbf{b}\in r}  Z(\mathbf{b})}
\end{eqnarray}
Here, $r$ is a random set of unit vectors. 

While $I_{c,\mathrm{rnd}}$ is non-deterministic, the random set $r$ can be chosen \textit{a priori} and used to calculate $I_{c,\mathrm{rnd}}$ for multiple clusters with the same dimensionality. As in calculations of $I_g$, we must calculate $I_c$ for many different clusters (Equations \ref{eq:I_gvec} and \ref{eq:I_grnd}). This means that in many circumstances it is more computationally efficient to calculate $I_{g,\mathrm{rnd}}$ than $I_{g,\mathrm{vec}}$ when $\mathcal{B}$ is pre-calculated (\Cref{tab:metricComparison},~\Cref{fig:timeVsDims}).  

The stochastic nature of this equation means that found values of $I_{c,\mathrm{rnd}}$ will slightly differ for different values of $r$. However, since the number of random unit vectors sampled is a hyper parameter of $I_{c,\mathrm{rnd}}$, the ability of $I_{c,\mathrm{rnd}}$ to approximate $I_{c,\mathrm{true}}$ can be improved by sampling more random unit vectors. As per~\Cref{thm:I_c_rnd_lim}, we can find: 
\begin{eqnarray}\label{eq:I_c_rnd_lim_I_c_true}
    \lim_{|r|\to \infty} I_{c,\mathrm{rnd}} = I_{c,\mathrm{true}}
\end{eqnarray}
We compare $I_{g,\mathrm{rnd}}$ to $I_{g,\mathrm{vec}}$ in two different contexts, to show an example use in the materials science domain, as well as a basic data science example to show the broader applicability.
\subsection{Isotropy of sets of clusters}\label{sets_of_clusters}
Both internal clustering validation metrics, and the measures for isotropy explored here rely on features of data to produce a numeric measurement. In order to adapt measurements of isotropy into an internal clustering validation metric, we extend these measurements to be defined globally for a set of clusters, $\mathcal{G}$, rather than a single cluster of points.
To adapt $I_{c,\mathrm{vec}}$ and $I_{c,\mathrm{rnd}}$ to estimate average $I_{c,\mathrm{true}}$ for a $\mathcal{G}$, a weighted sum of isotropy for each cluster can be taken to establish a measure for a global set of clusters, $I_g$ (weighted by the number of data points in a cluster). Thus $I_{g,\mathrm{vec}}$ can be defined as:
\begin{eqnarray}\label{eq:I_gvec}
    I_{g,\mathrm{vec}}(\mathcal{G}) \approx \frac{1}{\left|E\right|}\sum_{\mathcal{C}\in \mathcal{G}} |\mathcal{C}|I_{c,\mathrm{vec}}(\mathcal{C})
\end{eqnarray}
Where $E$ is the set of all points in the dataset ($|\mathcal{E}| = \sum_{C \in D}|\mathcal{C}|$, where $|\mathcal{C}|$ is the number of points in $\mathcal{C}$). 

Similarly we can define $I_{g,\mathrm{rnd}}$:
\begin{eqnarray}\label{eq:I_grnd}
 I_{g,\mathrm{rnd}}(\mathcal{G}) \approx \frac{1}{\left|\mathcal{E}\right|}\sum_{\mathcal{C}\in \mathcal{G}} |\mathcal{C}|I_{c,\mathrm{rnd}}(\mathcal{C}) 
\end{eqnarray}
 Both $I_{c,\mathrm{rnd}}$ and $I_{c,\mathrm{vec}}$ are bounded between 0 and 1 where 0 represents a set of clusters which are anistropic, and 1 represents a set of clusters which are isotropic.

 FA can also be adapted for a set of clusters $\mathcal{G}$. Taking the weighted sum of FA measurements for each cluster allows us to define $\mathrm{FA}_g$:
 \begin{eqnarray}\label{eq:FA_g}
 \mathrm{FA}_g(\mathcal{G}) = \frac{1}{\left|\mathcal{E}\right|}\sum_{\mathcal{C}\in \mathcal{G}} |\mathcal{C}|\mathrm{FA}(\mathcal{C}) 
\end{eqnarray}
$\mathrm{FA}_g$ is bounded between 0 and 1, with 0 representing a highly isotropic set of clusters and 1 representing a highly anisotropic set of clusters.
 
\section{Results}
\subsection{Use of isotropy measurements in the context of materials science}\label{materials_section}
To investigate the behaviour of $\mathrm{FA}_g$, $I_{g,\mathrm{vec}}$, and $I_{g,\mathrm{rnd}}$, we examine two potential use cases for these measures. We do this using a canonical dataset of crystal structures (\textit{e.g.}, the ICSD), and later in a more general context using a canonical data science dataset (\Cref{MNISTsection}). 

This exploration extends previous work~\cite{Durdy2022}, which applied \textit{K}-means clustering to prominent representations of the ICSD before and after the application of radial basis function (RBF) approximation~\cite{randomKitchenSinks}. This work qualitatively observed that clusters of chemical compositions in the ICSD were more isotropic after application of RBF approximation. This observation was made by visually inspecting 3 dimensional PCA projections of these high dimensional representations. These PCA projections inherently remove some of the information present in higher dimensions. Using the measures for isotropy of a set of clusters presented here, we are able to quantify the changes in cluster shape and ensure that all dimensions of a representation are considered.

In this work, we represent the ICSD using two popular composition based representations: a fractional composition vector ($compVec$) encoding~\cite{elemnet} and the $magpie$ composition based feature vector~\cite{Ward2016}. $compVec$ is an n-hot style encoding of composition where each entry in the vector corresponds to an element, the value of the entry represents the molar proportion of that element in a material. Thus, $compVec$ is a sparse representation, with each entry being between 0 and 1, and the sum of all entries being equal to 1. $magpie$ is a feature engineered vector, using features such as the covalent radius, electronegativity, or Mendeleev number observed for each element of a composition. The features are then aggregated using weighted mean, sum, variance, and range. As such, $magpie$ is a dense representation, with the range of features varying significantly (\Cref{diverenceInMeanDistToCentroid}), and many features being highly correlated~\cite{Durdy2022}. 

Min-Max scaling was used for both representations, to transform the feature values to be between -1 and 1 before these data were clustered using $K$-means clustering~\cite{kmeans} (values of $K$ set to 5 and 10). The RBF approximation was then applied \textit{a priori} to the data being clustered with $K$-means clustering (values of $K$ set to 5 and 10).

Previous work used RBF and RBF approximation interchangeably~\cite{Durdy2022}, despite them being distinct techniques~\cite{primerOnKernels}. The RBF provides an alternative distance metric to Euclidean distance for measuring the distance between points, while RBF approximation transforms points to emulate this distance metric while still using Euclidean distance. For brevity, we follow this convention in figures, but specify whether approximation is used in the text. We numerically measure the changes observed in previous work to validate visual findings and observe differences between measures of isotropy, $\mathrm{FA}_g$, $I_{g,\mathrm{rnd}}$ and $I_{g,\mathrm{vec}}$.

Visual inspection of PCA projections shows that applying an RBF approximation~\cite{randomKitchenSinks} to the ICSD creates more isotropic clusters (\Cref{fig:PCA}). However, due to the limitations in the visualisation, it is difficult to determine which representation leads to the largest changes upon application of the RBF approximation. The use of internal clustering validation metrics, including the isotropy measures developed here, allows us to quantify changes in the clusters (\Cref{tab:metricComparison}). 

Examining how metrics change with RBF approximation also allows reflection on which metric may be most useful depending on the required application. For example, cluster size variance was used in previous work, where more evenly sized clusters were sought to reduce data imbalance for training machine learning models~\cite{Durdy2022}. However, more evenly sized clusters are not necessarily well separated, thus other applications seeking to measure the distinctness of clusters may favour different internal cluster validation methods such as the Davies--Bouldin or silhouette scores. In the task of quantifying the anisotropy (\textit{i.e.}, spikeyness) of clusters, no existing internal validation methods were suitable. Thus, the methods presented here are pertinent in this case.

As was initially expected, applying RBF approximation has a consistent effect on most existing measurements in both investigated representations. The only internal cluster validation metrics in which $compVec$ and $magpie$ exhibited divergent behaviour with the RBF approximation were the mean Euclidean distance to the centroid, and measures for isotropy $\mathrm{FA}_g$, $I_{g,\mathrm{rnd}}$ and $I_{g,\mathrm{vec}}$ (\Cref{tab:ICSDMeasures}). We offer an explanation for each of these. The measures for isotropy introduced here ($\mathrm{FA}_{g}$, $I_{g,\mathrm{rnd}}$ and $I_{g,\mathrm{vec}}$) display unique behaviour from other internal cluster validation metrics. $I_{g}$ were the only metrics for which applying the RBF approximation $magpie$ and $compVec$ representations offered divergent behaviour regardless of whether measurements were taken after normalising $magpie$. 

While there are apparent divergent behaviours in the mean Euclidean distance between a point in a cluster and its centroid, this divergence can be explained by a lack of normalisation in the $magpie$ feature vector. We examine the formula of the RBF approximation applied in order to demonstrate this.

\subsubsection{Apparent divergent behaviour in mean Euclidean distance to centroid}\label{diverenceInMeanDistToCentroid}
When measuring the mean distance between a point in a cluster and its centroid after the RBF approximation, a $magpie$ representation without normalisation exhibits behaviour opposite to that of $compVec$. This behaviour can be explained by examining the formula for the radial basis function approximation:
\begin{eqnarray}
f(x) = \frac{ \sqrt{2} \cos(x\cdot w + o)}{\sqrt{l}} \label{eq:RBF} 
\end{eqnarray}
$l$ is length of vector. $w$ and $o$ are random weights, where $w$ is mean 0 variance $\sqrt{2 \gamma}$ and $o$ is uniform between 0 and 2 $\pi$.

The input to this function is randomly projected and translated, and then put through a cosine function before being scaled relative to the length of the input. 

In high dimensions, this random projection approaches a linear projection~\cite{Ritter1989,Kaski1998}. This linear projection maps into a data space in which distance relationships between points are preserved (with some error margin which scales inversely to the number of dimensions). If the inputs to this function are not normalised, there could be large differences between the scale of different axes in the new data space, which will be present after random projection. However, as the cosine function is bounded [-1,1], these large differences in scale of dimensions will be removed, akin to normalisation (\Cref{tab:ICSDMeasures}).

Unlike $compVec$, $magpie$ is not a normalised representation. Min-Max scaling was performed on the both representations before application of $K$-means clustering, however metrics were measured with this scaling (as this scaling would effect these measurements). For example, one of the features used in $magpie$ is the sum of the melting temperatures of the constituent elements, which ranges from 10~Kelvin to 43 million Kelvin. This large range means that the magnitude of $magpie$ feature vectors are on average very large (\Cref{tab:ICSDMeasures}), which results in large mean distances between a point in a cluster and that clusters centroid. After the RBF approximation the mean magnitude of the representation decreases, so in turn the mean distance from the centroid, this is in keeping with the RBF approximation having a bounded output domain (Equation~\ref{eq:RBF}). 

As all values in a $compVec$ representation range from 0 to 1, the mean magnitude is smaller (0.66). Applying an RBF approximation results in a mean magnitude which is larger, and in turn a larger mean Euclidean distance between a point and its centroid. Thus, the divergence of behaviour between representations on application of an RBF approximation is due to changes in the magnitude of a representation that this function entails.

\subsubsection{Divergent behaviour between representations for measurements of isotropy}
Applying an RBF approximation to all chemical compositions of the ICSD changes measurements of anisotropy, depending on the representation of those compositions (\Cref{tab:ICSDMeasures}). The $magpie$ representation is measured to be more isotropic after application of an RBF approximation, whereas the $compVec$ is more anisotropic. This difference in behaviour is due to differences in the sparsity of the representations. While $magpie$ is a dense representation, $compVec$ is a sparse representation, which has important consequences, discussed below.

Unlike the mean distance to the centroid, the divergence in the effect of the RBF approximation on isotropy cannot be explained by the bounded nature of the function (Equation~\ref{eq:RBF}). Even when measuring $I_{g,\mathrm{rnd}}$ and $I_{g,\mathrm{vec}}$ of this clustering when $magpie$ is normalised (with Min-Max scaling), the isotropy increases after the RBF approximation, whereas isotropy decreases in a $compVec$ representation (\Cref{tab:ICSDMeasures}).

As $compVec$ is a sparse representation, the projection seen in the RBF approximation will result in a sparse output. However, as this is over fewer dimensions, non-zero values in single axis have a greater effect on the isotropy of a cluster, leading to more anisotropic clusters.

As $magpie$ is a dense representation, the opposite is true. Variations in any one dimension will be diluted, leading to more isotropic clusters.

This behaviour is not necessarily obvious when looking at the transformation performed to both representations (Equation~\ref{eq:RBF}). In this case, the measures for isotropy $I_{g,\mathrm{vec}}$, $I_{g,\mathrm{rnd}}$, and $\mathrm{FA}_g$ are valuable to measure an effect and help identify changes that may not be intuitive.

\begin{table}[]
\caption{Analysis of entire ICSD in $compVec$ and $magpie$ representations before and after RBF approximation. Values here are means of those found with $K$-means clustering aplied with $K=5$ and $K=10$}
    \label{tab:ICSDMeasures}
    \centering


\begin{tabular}
{p{3cm}S[table-format=1.2e+2,table-column-width=1.7cm, tight-spacing=true,]S[table-format=1.2e+2,table-column-width=1.7cm, tight-spacing=true,]S[table-format=1.2e+2,table-column-width=1.7cm, tight-spacing=true,]S[table-format=1.2e+2,table-column-width=1.7cm, tight-spacing=true,]S[table-format=-1.2e+2,table-column-width=2.1cm, tight-spacing=true,]S[table-format=1.2e+2,table-column-width=1.7cm, tight-spacing=true,]}
\hline\hline
\multirow{2}{*}{Metric}& \multicolumn{2}{c}{compVec} & \multicolumn{2}{c}{magpie} & \multicolumn{2}{c}{magpie (normalised)} \\
 &    \multicolumn{1}{r}{No RBF} &    RBF  &     \multicolumn{1}{r}{No RBF} &    RBF  &                \multicolumn{1}{r}{No RBF} &    RBF  \\
\hline\hline
mean magnitude of representation ($l^2$-norm) & 0.660 & 1.01 & 8.21e+5 & 1.00  & 7.59 &  1.00\\[1cm]
mean distance to centroid &    0.397 &     0.52 & 2.33e+05 &    0.987 &                        2.91 &     6.98 \\ [0.45cm]
silhouette                 &    0.247 &    0.177 &    0.559 &  6.42e-3 &              -2.06e-2 &  6.34e-3 \\ [0.2cm]
Davies--Bouldin            &     1.91 &     2.58 &    0.515 &     8.33 &                        7.45 &     8.33 \\ [0.2cm]
Calinski--Harabasz         & 2.09e+04 & 1.88e+04 & 7.11e+05 &      8.53e2 &                    2.97e+03 &      8.53e2 \\ [0.45cm]
cluster size variance     & 9.26e+08 & 6.33e+08 & 6.34e+08 & 3.15e+05 &                    6.34e+08 & 3.15e+05 \\ [0.45cm]
fractional anisotropy & 0.855 & 0.870 & 0.994 & 0.182 & 0.951 & 0.182\\ [0.45cm]
$\mathrm{Var}(\lambda)$ & 3.33e-4 & 4.96e-4 & 1.08e-2 & 3.54e-6 & 1.24e-3 & 3.53e-6\\ [0.2cm]
$I_{g,\mathrm{vec}}$         &    0.942 &    0.923 &     0.180 &    0.993 &                       0.894 &    0.993 \\[0.2cm]
$I_{g,\mathrm{rnd}}$              &    0.992 &    0.988 &    0.682 &    0.999 &                       0.981 &    0.999 \\
\hline\hline

\end{tabular}

\end{table}

\begin{figure}
    
    \centering
    \begin{subfigure}{0.48 \textwidth}
    
    \includegraphics[width=\linewidth,trim=120 100 150 140,clip]{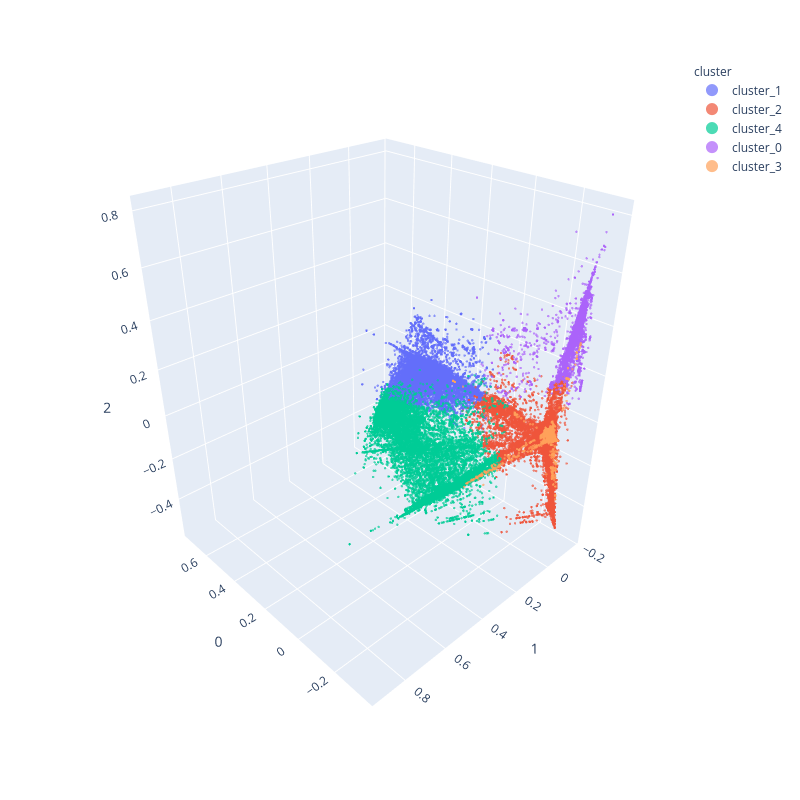}
    \caption{$CompVec$}
    \end{subfigure}
    \hfill
    \begin{subfigure}{0.48 \textwidth}
    
    \includegraphics[width=\linewidth,trim=120 100 175 155,clip]{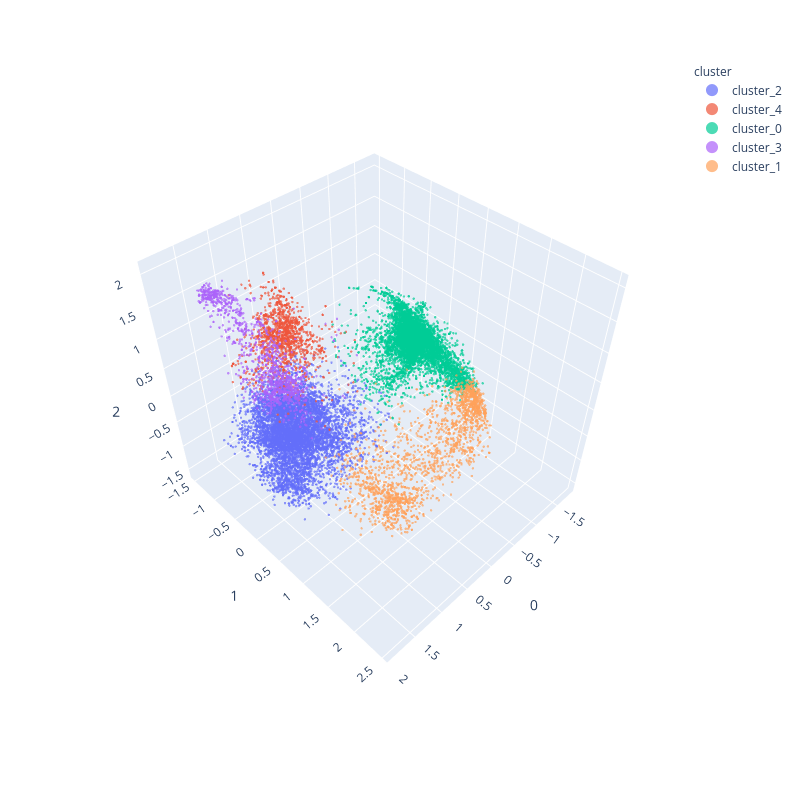}
    \caption{$CompVec$ with RBF}
    \end{subfigure}
    \begin{subfigure}{0.48 \textwidth}
    \includegraphics[width=\linewidth,trim=90 90 130 110,clip]{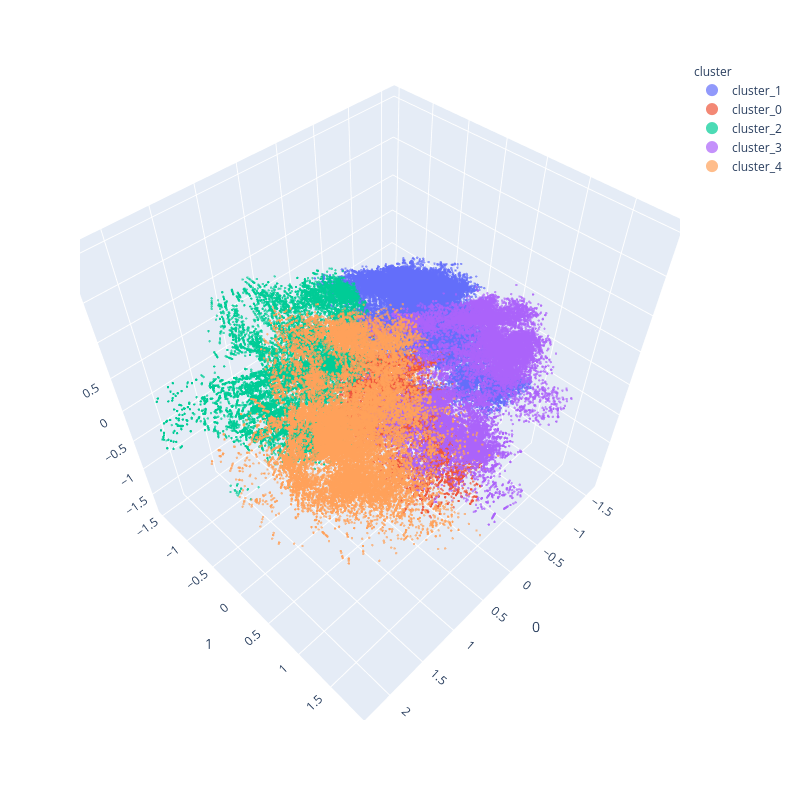}
    \caption{$Magpie$}
    \end{subfigure}
    \hfill
    \begin{subfigure}{0.48 \textwidth}
    
    \includegraphics[width=\linewidth,trim=110 100 130 130,clip]{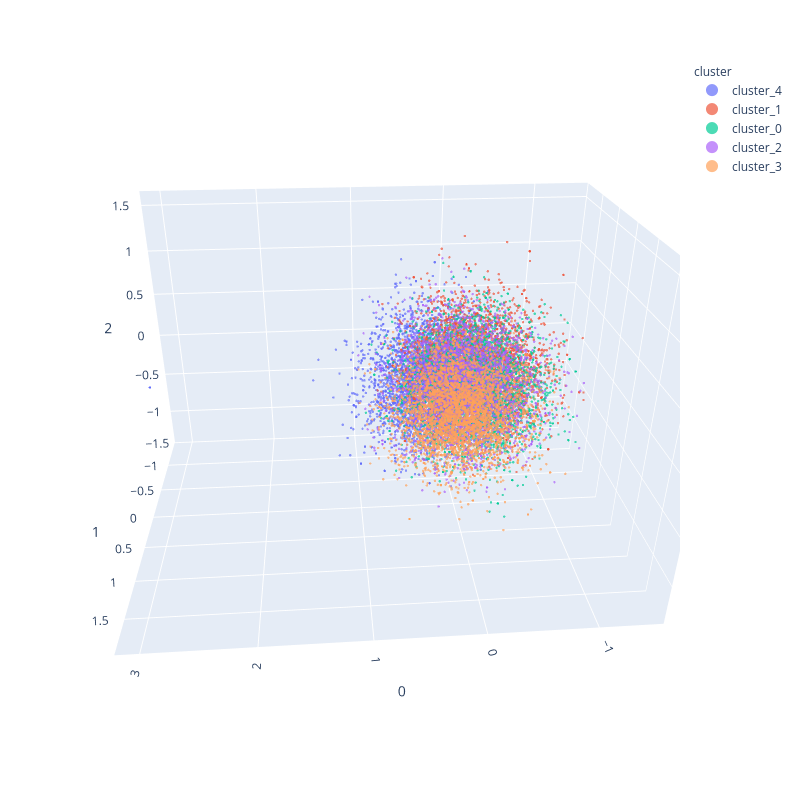}
    \caption{$Magpie$ with RBF}
    \end{subfigure}
    \caption{3 dimensional PCA representations of the ICSD (a random 20\% subsample is used for visual clarity) with clusters found with $K$-means clustering. Coloured according to clusters found using $K$-means clustering on that data with $K=5$ (i.e. the same point may not be the same colour between subfigures) (a)~Data was represented using $compVec$ with no RBF application. (b)~Data was represented using $compVec$ with RBF application. (c)~Data was represented using magpie with no RBF application. (d)~Data was represented using magpie with RBF application.}
    \label{fig:PCA}
\end{figure}
\subsection{Example basic use of cluster isotropy measurements for data science application}\label{MNISTsection}
\begin{figure}
    \centering
    \includegraphics[width=\textwidth]{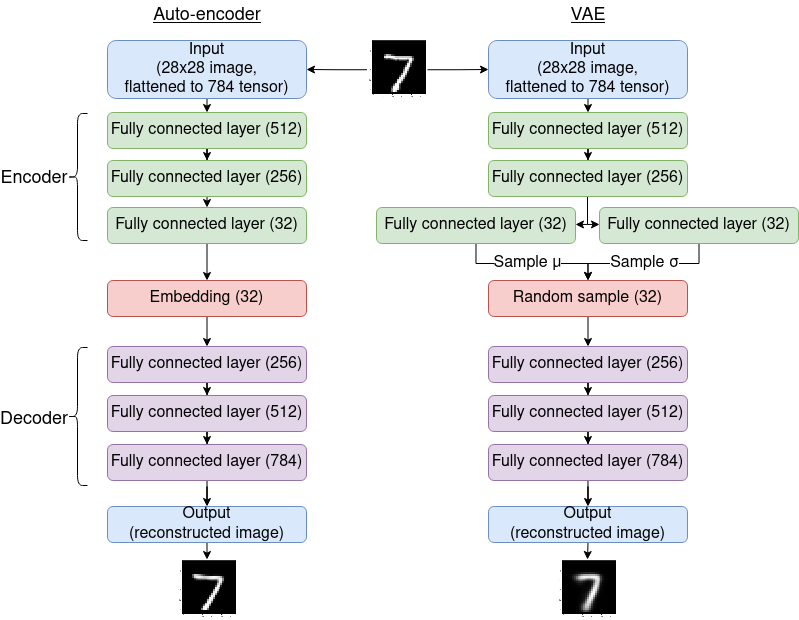}
    \caption{Layout of auto-encoder and variational auto-encoder (VAE) used in the example discussed in ~\ref{MNISTsection}. Embeddings shown in red here are those seen plotted in~\Cref{fig:mnist_embeddings} and measured in~\Cref{tab:MNISTMeasures}.}
    \label{fig:model_layout}
\end{figure}
\begin{figure}
    \centering
    \begin{subfigure}{0.48\linewidth}
        \includegraphics[width=\linewidth,trim=40 40 10 0,clip]{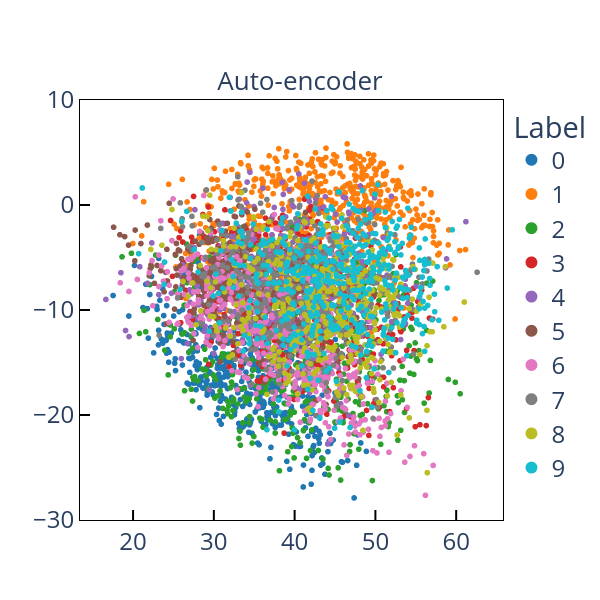}
        \caption{\centering}
    \end{subfigure}
    \hfill
    \begin{subfigure}{0.48\linewidth}
        \includegraphics[width=\linewidth,trim=40 40 10 0,clip]{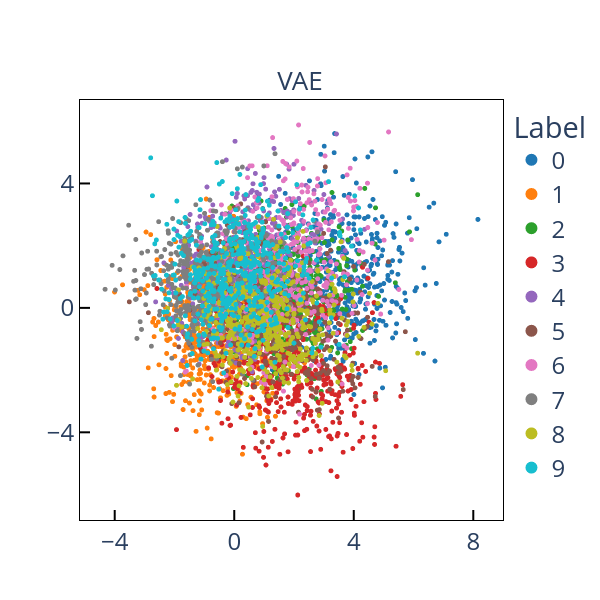}
        \caption{\centering}
    \end{subfigure}
    
    \caption{PCA transformations of latent space embeddings of MNIST training set. A random 50\% of the training set is shown for visual clarity. (a) Autoencoder embedding (b) Variational autoencoder embedding}
    \label{fig:mnist_embeddings}
\end{figure}
Examining learnt embeddings (also called learnt representations) is common in understanding how deep learning algorithms interpret input data~\cite{mat2Vec}. However, this is often qualitative analysis, based 2D or 3D projections of the embeddings which by their nature remove some of the information present in higher dimensions. 

Cluster isotropy measurements can be used to quantify differences between embeddings without dimensionality reduction. This can be used in conjunction with qualitative analysis to judge differences between embeddings produced by deep learning models.

In order to demonstrate the broader applicability of measures of cluster isotropy, we present quantitative analysis of learnt embeddings. Two models were trained on a basic computer vision dataset, the Modified National Institute of Standards and Technology dataset (MNIST)~\cite{mnist}. The MNIST is a widely used dataset of 70,000, 28x28 pixel images, each depicting a handwritten number between 0 and 9. A common ML introductory task is to create a model that can correctly classify these images. 

We train two types of models to create low dimensional embeddings MNIST images, and compare their embeddings of unseen test data using two different models. The exact mechanics of these models is not vital to understand the novel analysis of isotropy of clusters which is presented here. We do not aim to present a novel model in this paper; instead we use isotropy of clusters to analyse models' outputs in an interesting way. Regardless a brief overview is of these models given.

Auto-encoders are neural network models that learn a representation of a dataset by passing each dataset through a neural network with an information bottleneck before trying to reconstruct the input data point~\cite{autoencoder}. The section of the model up to and including this bottleneck is called the encoder, and the section after the bottleneck is called the decoder (as they, respectively, encode and decode a learnt representation). Once trained, output of the bottleneck layer can be used as a lower dimensional representation of the input (also referred to as a learnt or latent space). Feeding random noise to a trained auto-encoder should generate an output which is similar to existing data points, but is completely fictional. However, in practise, when used in this generative fashion, auto-encoders can give outputs identical to received training data. 

Variational auto-encoders (VAEs) aim to be a more useful generative model by introducing randomness into the encoder~\cite{VAE_review}. This forces the model to learn the distribution of points in a dataset, rather than the points themselves. As a result, a VAE tends to generate a data point which is harder to map on directly onto a data point in the training set. In order to stop overfitting when learning the distribution of datapoints a further term, Kullback–Leibler divergence (KL divergence) is introduced into the loss function~\cite{kl_divergance}. The KL divergence measures the distance between the learnt distribution and either the true distribution or the a distribution chosen a priori (often the normal distribution).

An auto-encoder and a VAE~(\Cref{fig:model_layout}), were trained for 100 epochs on the MNIST dataset. Each features 3 fully connected layers in both the encoder and the decoder, however, with rectified linear units (ReLUs) after each layer to provide non-linearity (with the exception of the final layer which is followed by a sigmoid function). The KL divergence in the VAE was measured between the learnt distribution and a normal distribution of mean 0 and standard deviation 1.

The encoder of each network produces embeddings of size 32 for an input image. The latent space of these training set can be projected into 2 dimensions using principal component analysis (PCA) to allow for inspection (\Cref{fig:mnist_embeddings}). 

On visual inspection it is difficult to discern differences between these latent spaces. Measurements of $\mathrm{FA}_g$, $I_{g,\mathrm{vec}}$, and $I_{g,\mathrm{rnd}}$ allow us to quantitatively say that images of the same label form more isotropic clusters when embedded with the VAE than when embedded with the auto-encoder~(\Cref{tab:MNISTMeasures}). All three metrics for isotropy of clusters conclude that VAE's result in more isotropic clusters. FA$_g$ of the VAE's is measured as almost half that of the AutoEncoder, suggesting a very large change in isotropy. $I_{g,\mathrm{rnd}}$, and $I_{g,\mathrm{vec}}$ suggest a much smaller (but consistent) change in isotropy. However, as we discuss later (\Cref{discussion}) this seems to be indicative of how these metrics perform, with changes of $\sim$0.001 being notable despite only being a very small part of the domain of $I_{g,\mathrm{rnd}}$ and $I_{g,\mathrm{vec}}$ (which is 0, 1). 

Conclusions as to the difference between these latent spaces can also be drawn from other internal cluster validation metrics. The Calinski--Harabasz measure suggests that clusters in the VAE's embeddings are more dispersed than those of the Auto-Encoder. This is in line with the silhouette score, which suggests that points embedded with VAE's are closer to points in other clusters than those embedded with auto-encoders. Davies--Bouldin also suggests more poorly separated clusters in the VAE embeddings. 

Existing internal cluster validation metrics all suggest a worse separation between classes in VAE embeddings compared to Auto-Encoder embeddings. This makes sense when considering how VAEs work; Gaussian noise is inherent to the model. For generative models such as these, isotropy in the embeddings has been linked to good generative performance~\cite{isotropy}. Thus, measurements of isotropy for sets of clusters are pertinent.

Using existing internal cluster validation metrics may suggest to a researcher that VAE's are worse for class separation than Auto-Encoders. Measuring the isotropy of embeddings (and knowing that isotropy has been associated with good generative performance) gives a more nuanced picture. When observed with other internal clustering validation measures, metrics for the isotropy of clusters allows someone developing models such as these to gain an intuitive understanding of the latent spaces produced by them.

\begin{table}[]
\caption{Metric measurements of the MNIST test set embedded with auto-encoder and VAE}
\label{tab:MNISTMeasures}   
    \centering

\begin{tabular}{c c c}
\hline\hline
    Metric & Auto-encoder  & VAE  \\\hline \hline
     $I_{g,\mathrm{vec}}$ & 0.991 & 0.998 \\
     $I_{g,\mathrm{rnd}}$ & 0.942 & 0.984 \\
     $\mathrm{Var}(\lambda)$ & $4.70\times10^{-3}$ & $2.76\times10^{-4}$\\
     fractional anisotropy &  0.901 & 0.451 \\
     Calinski--Harabasz & 531 & 161\\
     Davies--Bouldin &  2.57 & 7.03 \\
     Silhouette & 0.091 & 0.011 \\
     
     \hline\hline
\end{tabular}

\end{table}
\subsection{Examining differences between measurements of isotropy using random Gaussian point clouds}\label{dummy_data_experiments}
\begin{figure}
    \begin{subfigure}{0.45\linewidth}
    \centering
    \includegraphics[width=\linewidth]{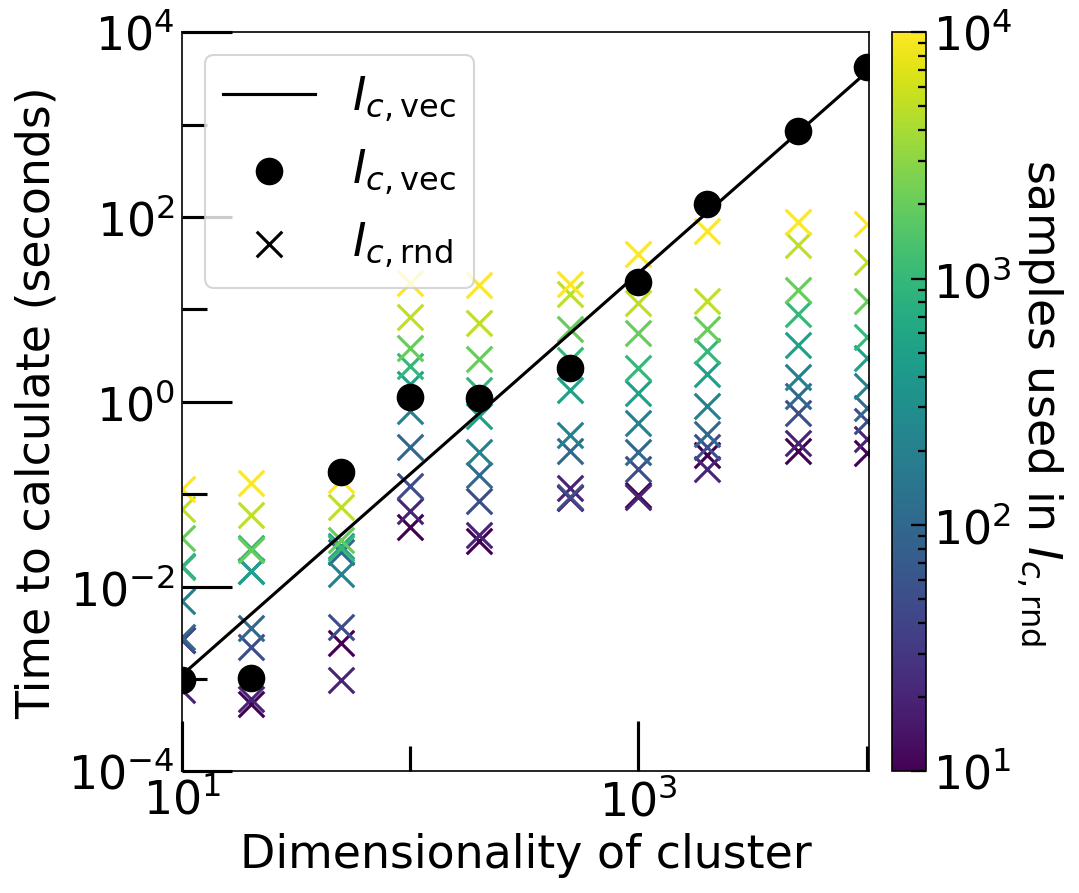}
    \caption{\centering}
    \label{fig:timeVsDims}
\end{subfigure}
    \centering
    \begin{subfigure}{0.46\linewidth}
        \includegraphics[width=\linewidth]{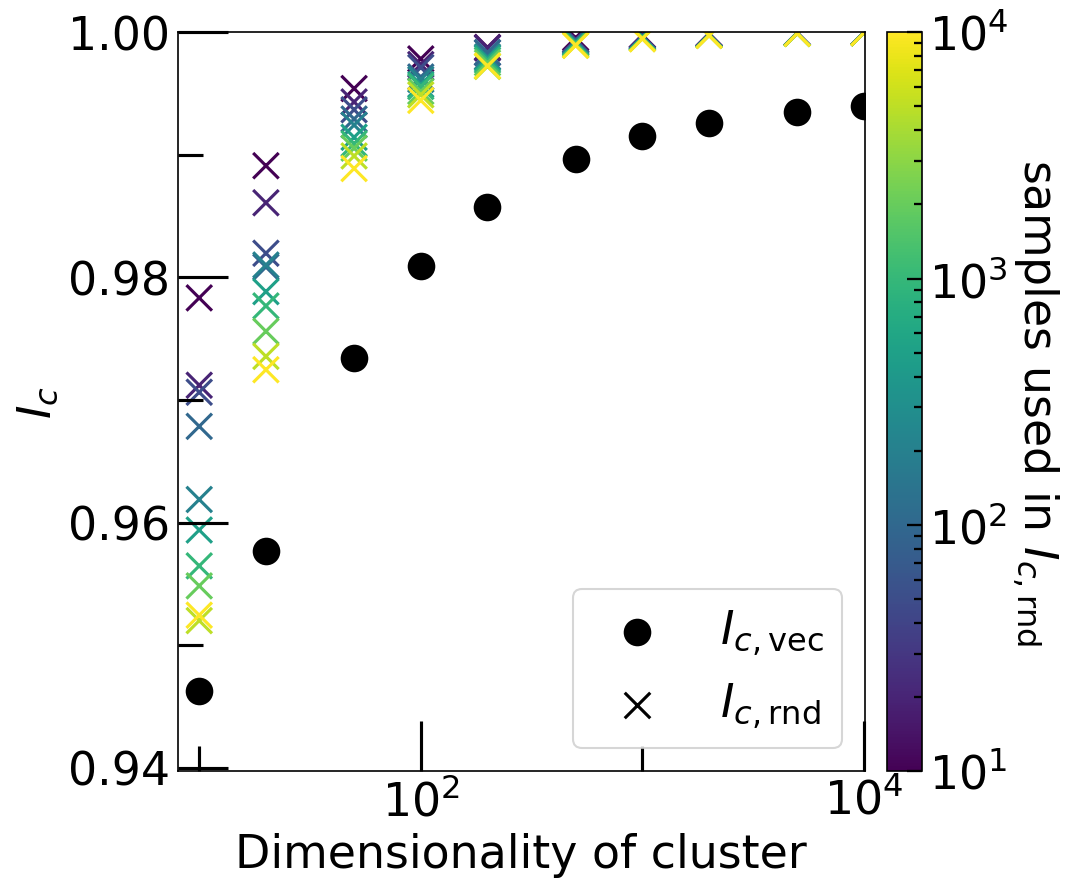}
        \caption{\centering}
        \label{fig:dimVsPerformance}
    \end{subfigure}
    
    \caption{Measurements of $I_c$ using $I_{c,\mathrm{vec}}$ and $I_{c_rnd}$, across a random Gaussian cluster of different dimensions. This was repeated for 10 different random clusters of 100 points, with the mean results shown here (a) In higher ($>10^3$) dimensions, $I_{c,\mathrm{vec}}$ becomes very expensive to compute, the computational complexity $I_{c,\mathrm{rnd}}$ scales with the number of unit samples taken. (b) When the cluster is 10 or fewer dimensions, $I_{c,\mathrm{rnd}}$ is a more accurate, with $I_{c,\mathrm{vec}}$ performing better in higher dimensions. As expected (Equation~\ref{eq:I_c_rnd_lim_I_c_true}), $I_{c,\mathrm{rnd}}$ measurements are more accurate when more unit samples are used, though this effect is less noticeable in higher dimensions. Regardless of the number of unit samples used, across all dimensions all measurements of $I_{c,\mathrm{vec}}$ and $I_{c,\mathrm{rnd}}$ are within 10\% of each other.}
    \label{fig:dummyDataGraphs}
\end{figure}

\begin{figure}
\centering
    \begin{subfigure}{0.45\linewidth}
        \includegraphics[width=\linewidth]{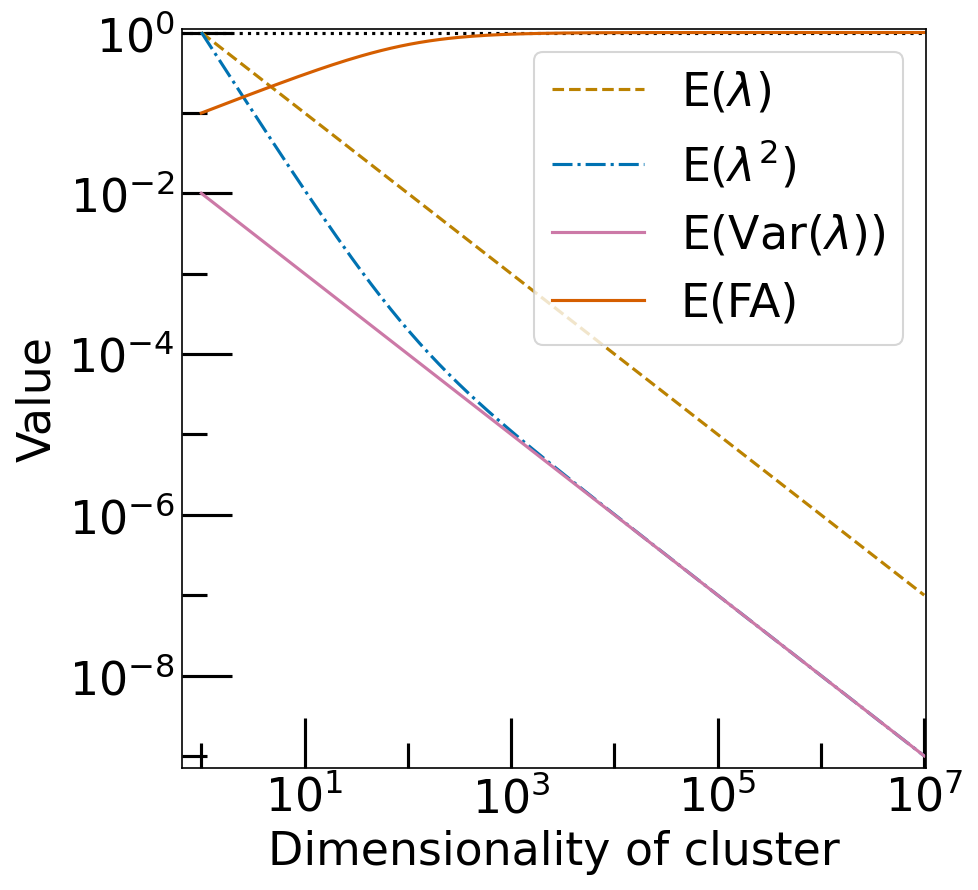}
        \caption{\centering}
        \label{fig:MP_dists}
    \end{subfigure}
    \begin{subfigure}{0.45\linewidth}
    \centering
    \includegraphics[width=\linewidth]{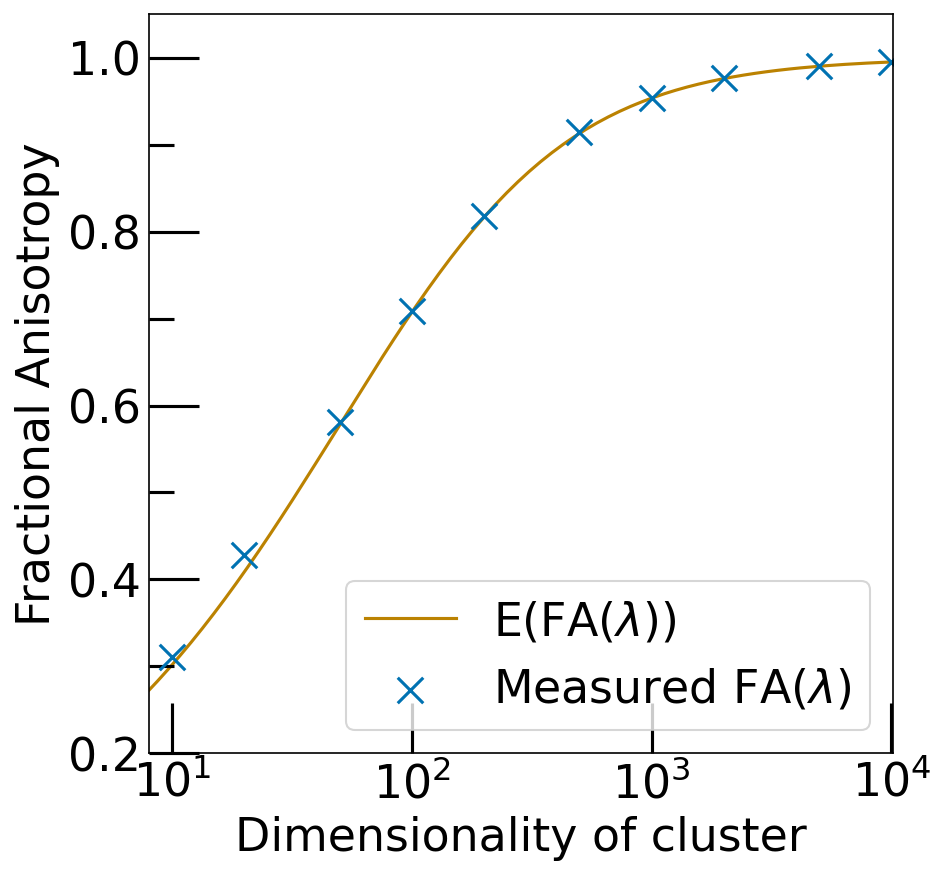}
    \caption{\centering}
    \label{fig:FAvsDims}
\end{subfigure}
    
      \begin{subfigure}{0.45\linewidth}
        \includegraphics[width=\linewidth]{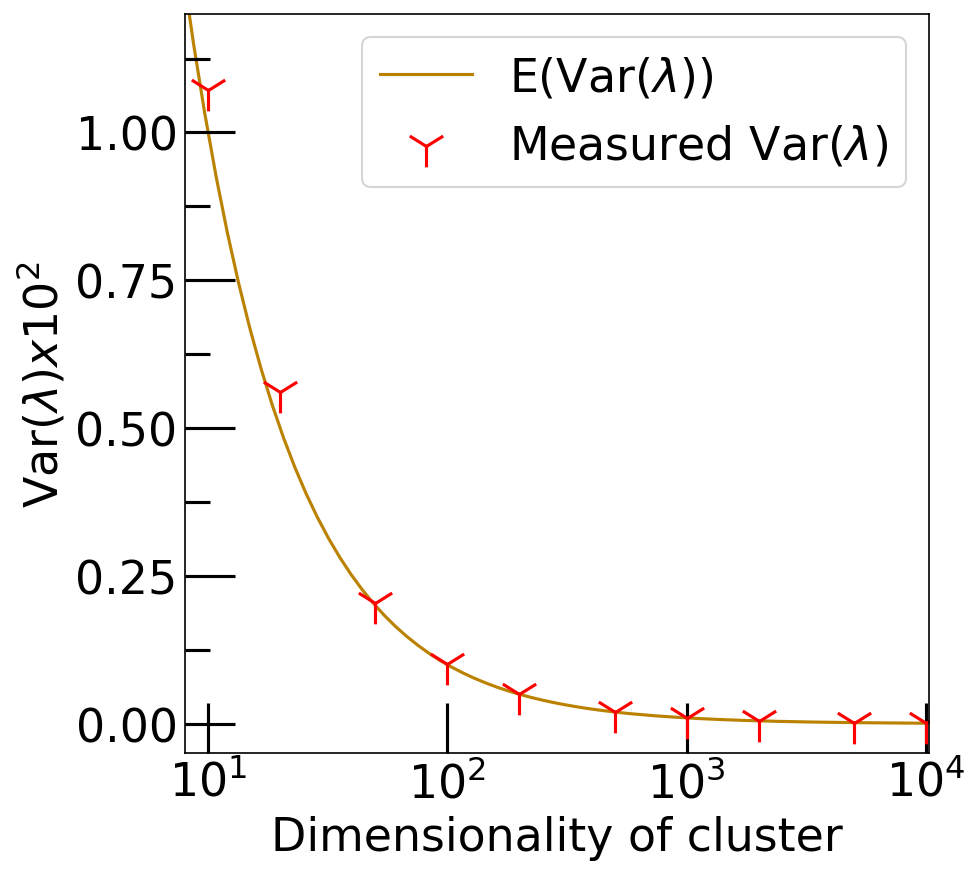}
        \caption{\centering}
        \label{fig:var_lambdavs_dims}
    \end{subfigure}
    \caption{(a) The expected value of fractional anisotropy (FA) for clusters of 100 points in different dimensionalities. $\lambda$ represents the normalised eigenvalues for each cluster (Equation 3.1). The coordinates of each point in each cluster are sampled from the normal distribution $\sim\mathcal{N}(0,1)$. 
    This $E$(FA) is calculated using the Marchenko--Pastur distribution. $E(\mathrm{FA})$ and $E(\mathrm{Var}(\lambda))$ change with the dimensionality of a cluster. Consequently, clusters of different dimensionality cannot be compared using these measures. 
    When comparing high dimensional clusters, $\mathrm{FA}$ and $\mathrm{Var}(\lambda)$ may lead to counterintuitive results, particularly when data are noisy. 
    (b) Measurements of FA compared to the expected value modelled with the Marchenko--Pastur distribution. 
    Random matrix theory successfully describes the behaviour of this measure. Points represent the mean FA of 10 different random clusters of 100 points. 
    (c) Measurements of Var($\lambda$) compared to the expected value as modelled with the Marchenko--Pastur distribution. Points represent the mean Var($\lambda$) of 10 different random clusters of 100 points.}
    \label{fig:dummyDataGraphs_fa}
\end{figure}

In experiments presented thus far, isotropy of sets of clusters have been measured in two different situations (\textit{i.e.}, Auto-encoder vs VAE, RBF vs No RBF). In both experiments, all measures of isotropy have been consistent as to which of the two situations results in more isotropic clusters (\Cref{tab:ICSDMeasures,tab:MNISTMeasures}) (note that a $FA$ is expected to have an inverse relationship to $I_c$). However, it has not been clear which approximation of $I_{c}$ was more accurate and if there were any advantages or disadvantages to approximating $I_{c}$ as opposed to using FA.

To further explore the differences between the approximations of $I_c$ and FA, the isotropy of randomly generated clusters of different dimensionalities were measured. Measurements of $I_{c,\mathrm{vec}}$, $I_{c,\mathrm{rnd}}$, and FA were taken to explore how the accuracy, and time complexity of each algorithm varies with dimensionality. 

Clusters of 100 points were generated, with each point having coordinates sampled from a Gaussian distribution of mean 0 and standard deviation 1. Clusters were generated between 10 and 10,000 dimensions and the $I_{c,\mathrm{vec}}$, $I_{c,\mathrm{rnd}}$, and FA were measured, with the number of random unit vectors used in $I_{c_rnd}$ varying between 10 and 10,000. For each dimensionality, 10 different clusters were measured; we report the mean results here.

The computational times of $I_{c,\mathrm{vec}}$ and FA are approximately exponential to the dimensionality of the cluster (\Cref{fig:timeVsDims}). As the computational time of $I_{c,\mathrm{rnd}}$ varies with the number random unit vectors used, $I_{c,\mathrm{rnd}}$ often performs faster (particularly for high dimensional clusters). This is generally in line with the complexity calculated for these algorithms (\Cref{tab:metricComparison}).

In line with expectations (Equation~\ref{eq:I_c_rnd_lim_I_c_true}), as more random samples are used, $I_{c,\mathrm{rnd}}$ becomes more accurate (\Cref{fig:dimVsPerformance}). Even with the highest number of unit samples $I_{c,\mathrm{rnd}}$ was never measured as being more accurate than $I_{c,\mathrm{vec}}$ for these Gaussian clusters. Both $I_{c,\mathrm{rnd}}$ and $I_{c,\mathrm{vec}}$ measurements increased slightly with dimensionality. This is inverse to the relationship found with FA (\Cref{fig:FAvsDims}), which would imply that in high dimensions Gaussian clusters are highly anisotropic. The reason for FA's observed increase in this experiment can be found in the field of random matrix theory, and are explored below (\Cref{random_matrix_theory})

\subsubsection{Why does fractional anisotropy of a Gaussian cluster increase with dimensionality?}\label{random_matrix_theory}

Gaussian clusters are observed to have a high FA (and thus be considered very anisotropic) in higher dimensions (\Cref{fig:FAvsDims}). If FA is highly correlated with with dimensionality, does this mean FA is not applicable in higher dimensions? In order to answer this, we calculate how the expected value of FA for a Gaussian cluster varies with dimensionality of that cluster.

To examine the expected value of FA of a point cloud with coordinates sampled from the Gaussian distribution, one can consider this $n$ dimensional point cloud of $T$ points to be represented as a matrix size $n \times T$. The eigenvalues for the principal components of this random matrix can be described using a probability density function. Assuming the variance of our Gaussian distribution, $\sigma^2$, to be finite, the probability density function of the eigenvalues, $\Lambda$, is given by the Marchenko--Pastur distribution~\cite{MarcenkoPastur}:

\begin{eqnarray}\mathrm{PDF}_{\mu,\sigma^2}(\Lambda_x)=\frac{1}{2\pi\sigma^2\Lambda}\sqrt{(\Lambda_{\max}-\Lambda)(\Lambda-\Lambda_{\min})}\end{eqnarray}

where $\lambda_{\max}$ and $\lambda_{\min}$ are respectively the largest and smallest possible eigenvalues of the distribution, given by:

\begin{eqnarray}\Lambda_{\max,\min}\approx\sigma^2\left(1\pm\sqrt{\frac{T}{n}}\right)^2+\mu\end{eqnarray}

Where, $\mu$ is the mean of the distribution and $\sigma^2$ is the variance. To solve this for the parameters of our experiment ($\mu = 0, T=100, \sigma=1$) it can be said:

\begin{eqnarray}\Lambda_{\max,\min}=\left(1\pm\frac{10}{\sqrt{n}}\right)^2\end{eqnarray}

In order to find the expected value for the E(FA), we can calculate $\mathrm{E}(\Lambda^2)$, and $\mathrm{E}(\Lambda)^2$, the Marchenko--Pastur distribution can be integrated: 

.\begin{eqnarray}E(\Lambda)=\int_{\Lambda_{\min}}^{\Lambda_{\max}}\Lambda \mathrm{PDF}_{\mu,\sigma^2}(\Lambda) d\Lambda\end{eqnarray}

Thus, in our case:

\begin{eqnarray}E(\Lambda) &=\int_{\Lambda_{\min}}^{\Lambda_{\max}}\Lambda \frac{1}{2\pi\sigma^2\Lambda}\sqrt{(\Lambda_{\max}-\Lambda)(\Lambda-\Lambda_{\min})} d\Lambda \\[0.15cm]
&= \int_{\Lambda_{\min}}^{\Lambda_{\max}} \frac{1}{2\pi}\sqrt{(\Lambda_{\max}-\Lambda)(\Lambda-\Lambda_{\min})} d\Lambda \end{eqnarray}

Similarly we can say 

\begin{eqnarray}E(\Lambda^2)= \int_{\Lambda_{\min}}^{\Lambda_{\max}} \frac{\Lambda}{2\pi}\sqrt{(\Lambda_{\max}-\Lambda)(\Lambda-\Lambda_{\min})} d\Lambda \end{eqnarray}

FA is calculated with the normalised eigenvalues, $\lambda$ (Equation~\ref{eq:normalised_lambda}), and we can say:
\begin{eqnarray}
E(\lambda) = \frac{E(\Lambda)}{T}
\end{eqnarray}
and 
\begin{eqnarray}
E(\lambda^2) = \frac{E(\Lambda^2)}{T^2}
\end{eqnarray}
Note that when substituting these into the equation for FA($\lambda$) (Equation~\ref{eq:FA_deffinition}), the denominators cancel out, thus $E(\mathrm{FA}(\lambda$)) is equal to $E(\mathrm{FA}(\Lambda$)): 
\begin{eqnarray}
E(\mathrm{FA}(\lambda)) &=
\sqrt{1-\frac{E(\lambda)^2}{E(\lambda^2)}} \\[0.15cm]
&= \sqrt{1-\frac{ \left( E(\Lambda)/T \right)^2}{E(\Lambda^2)/T^2}} \\[0.15cm]
&= \sqrt{1-\frac{E(\Lambda)^2/T^2}{E(\Lambda^2)/T^2}} \\[0.15cm]
&= \sqrt{1-\frac{E(\Lambda)^2}{E(\Lambda^2)}}\\[0.15cm]
&= E(\mathrm{FA}(\Lambda))\end{eqnarray}

Plotting E($\lambda$) against E$(\lambda^2)$ shows that they converge in high dimensions, and thus FA converges to 1 (\Cref{fig:MP_dists}). The expectation of FA to measure Gaussian clusters as being anisotropic in high dimensions was reflected in our measurements (\Cref{fig:FAvsDims}). A similar process can be followed for Var($\lambda$), showing that Var($\lambda$) trends towards 0 in high dimensions for Gaussian clusters (\Cref{fig:var_lambdavs_dims}).

The implication of this proof is that comparing FAs for clusters that are not represented in the same number of dimensions may not be applicable. This is particularly true where clusters are noisy and thus will have Eigenvectors expected to more closely follow the Marchenko--Pastur distribution. As will be discussed (\Cref{discussion}) this does not mean that FA is not a useful tool, but that this result should be taken into consideration, for example, by avoiding the use of noisy data or checking conclusions drawn from use of FA agree with those that would be drawn from the use of $I_{c,\mathrm{vec}}$ or $I_{c,\mathrm{rnd}}$.

\section{Discussion}\label{discussion}
Analysis of isotropy of sets of clusters has used in two different settings to provide quantitative evidence that enhances observations of low dimensional PCA projections. The study of isotropy of a set of clusters has highlighted non-intuitive results (\textit{e.g.}, that RBF approximations will result in changes in isotropy that depend on the sparsity of an input representation), and helped interpretation of learnt embeddings, where low-dimensional projections were unclear. 

In real world examples seen here, measurements for isotropy have agreed in all instances. All measurements for isotropy investigated have a bounded output domain, and they are invariant under linear isometries, as well as under uniform scaling. Thus, any of Var($\lambda$), FA$_g$, $I_{g,\mathrm{vec}}$ or $I_{g,\mathrm{rnd}}$ can prove useful, and will likely summon similar conclusions for data. 

One way in which these measures differed is in the use of their output domain. As previously noted while Var($\lambda$) has a theoretical output domain of 0, 0.25, in practice the highest Var($\lambda$) observed was 0.033 in synthetic data (\Cref{fig:var_lambdavs_dims}) or 0.011 in real data. As such, and because of its lack of appearance in the literature, Var($\lambda$) not the focus of this study. $I_{g,\mathrm{rnd}}$ and $I_{g,\mathrm{vec}}$ also use a small amount of their output domain (0, 1) in the examples measured; in most cases seen $I_{c}>0.9$. 

In contrast to $I_{c,\mathrm{rnd}}$, $I_{c,\mathrm{vec}}$ and Var($\lambda$), FA is observed across most of its output domain (\Cref{tab:MNISTMeasures,fig:dummyDataGraphs_fa}). However, when measuring a Gaussian cluster, the nature of random matrices means that high dimensional measurements of FA will seem anisotropic (\Cref{random_matrix_theory}). While this does not universally preclude the applicability of FA in high-dimensional clusters, where data are noisy, this may provide a measurement that does not align with intuitive understandings of isotropy.

One intuition about isotropy, which is common in the existing literature is the use of principal components to identify axis of anisotropy~\cite{isotropy,Fractional_Anisoptry}. The implementation introduced here, $I_{c,\mathrm{rnd}}$ does not follow that intuition, yet still performs well as a measure of isotropy. In Gaussian clusters $I_{c,\mathrm{rnd}}$ is a less accurate approximation of $I_{c,\mathrm{true}}$ than $I_{c,\mathrm{vec}}$. However, in both real-world and synthetic settings, examples have been observed in which $I_{c,\mathrm{rnd}}$ is more accurate than $I_{c,\mathrm{vec}}$ (\Cref{fig:two_clusters,tab:MNISTMeasures}). While differences between sizes of principal components can be indicative of anisotropy, a lack of data between these principal components can also contribute to anisotropy (\textit{e.g.},~\cref{fig:two_clusters}). By prioritising measurements on principal axes, existing measures deprioritise the latter of these causes of anisotropy, which can lead to unintuitive results.

Some aspects of whether a cluster of points is isotropic can be subjective (as discussed in the Introduction). We do not argue that any of the metrics for isotropy examined here are the best, instead examining their differences and observing cases which may give unexpected results (\textit{e.g.}, FA being sensitive to noise in high dimensions). Describing the shape of a cluster in a single number will by necessity remove detail and is no substitute for visual inspection. This is even more the case when trying to describe sets of clusters. 

However, when working with unlabelled data (as is often the case in materials science), visual inspection may prove impractical. Clusters may be hard to visually separate, or there may be too many to feasibly inspect each of them. Thus, while simplistic, the tools explored here may be helpful in a number of settings.


In materials science, the introduction of new clustering techniques~\cite{ILS}, provides exciting ways to apply clustering algorithms in ways which suit the unique aspects of materials science data. However, if these clustering algorithms are dependant on distance metrics such as the Euclidean distance, then they are also dependant on materials representation. As there is no definitive representation of a material, it is important to have metrics to analyse resulting clusterings. Where no target labels exist, the metrics available for such analysis are limited. The isotropy of clusters is linked to downstream performance benefits~\cite{isotropy}, and is particularly pertinent for analysing different representations. Dominant features, or correlations of features, may not be apparent in lower dimensions (\Cref{fig:PCA}) and identifying how such correlations may affect clustering is important when drawing conclusions from data or choosing representations of data.

\section{Conclusion}
Analysing sets of clusters is a common task both in the materials science and broader data science domains. While metrics (referred to as internal cluster validation metrics) exist to analyse the compactness and separability of unlabelled clusters, analysis of the shape of clusters has till now been qualitative. 

This research offers a thorough exploration of metrics for isotropy (\textit{i.e.} spikiness) of a cluster of points. One such metric, FA was demonstrated on higher dimensional data. Through use of theorems from the field of random matrix theory~\cite{MarcenkoPastur}, we demonstrated that in high dimensions this measure is susceptible to giving more anisotropic results than expected for noisy data.

A separate measure for isotropy of a point cloud, $I_{c,\mathrm{vec}}$, was examined. An alternative implementation, $I_{c,\mathrm{rnd}}$ was proposed based on the existing derivation of $I_{c,\mathrm{vec}}$. The differences between $I_{c,\mathrm{vec}}$ and $I_{c,\mathrm{rnd}}$ were discussed. Neither $I_{c,\mathrm{rnd}}$, nor $I_{c,\mathrm{vec}}$ was seen to be universally more accurate or faster to compute. However, for high dimensional data, $I_{c,\mathrm{rnd}}$ is orders of magnitude faster to compute than $I_{c,\mathrm{vec}}$

$I_{c,\mathrm{rnd}}$, $I_{c,\mathrm{vec}}$, FA and a basic proxy for isotropy, Var($\lambda$), were generalised to measure isotropy unlabelled clusters of data. We demonstrate two real-world applications of these generalisations: One in the materials informatics domain, and one in a broader data science domain. 

In materials science, understanding the output to clustering is particularly important. Datasets are often heterogeneous, and clustering algorithms tend to perform poorly. We demonstrate the usefulness of isotropy measures for clusters of data by exploring clusters found in the ICSD, a canonical materials science dataset. Previous research qualitatively described these clusters using low dimensional visualisations, we quantify these findings by measuring isotropy numerically.

Broader data science applicability of these measures are demonstrated on by analysing learnt representations of a fundamental data science dataset (MNIST). Isotropy measures for sets of clusters allow for more thorough exploration and description of these representations than would otherwise be possible. 

Internal cluster validation measures are helpful tools for chemists and data scientists to understand and quantify unlabelled sets of clusters. Isotropy is pertinent to machine learning for materials science as appropriate material representation is often unclear, and anisotropy in a cluster can be indicative of dominant features in a representation of a material. 

 We provide implementations of these metrics in the associate code repository~\cite{isotropy_git}. The metrics for isotropy presented here are a helpful addition to existing metrics, which allow researchers to richly explore their datasets.

\section*{Acknowledgements}
We thank the Leverhulme Trust for funding this work via the Leverhulme Research Centre for Functional Materials Design. MWG thanks the Ramsay Memorial Fellowships Trust for funding through a Ramsay Trust Memorial Fellowship.
\section*{Conflict of interest}
There are no competing interests to declare.
\section*{References}
\bibliographystyle{iopart-num}
\bibliography{sample}
\appendix

\renewcommand{\theequation}{A.\arabic{equation}}
\section*{Appendix}
\newtheorem{theorem}{Theorem}

\begin{theorem}\label{thm:varianceProof}
For a finite sized set, $\lambda$, of real numbers between 0 and 1, the variance of $\lambda$ is bounded between 0 and 0.25: $$0\le Var(\lambda)\le 0.25$$
\end{theorem}
\begin{proof}[Proof of \ref{thm:varianceProof}]
Let $\lambda={\lambda_1,\lambda_2,\dots,\lambda_n}$ be a set of real numbers such that $0 \leq \lambda_i \leq 1$ for all $i$. Let $\bar{\lambda}=\frac{1}{n}\sum_{i=1}^n \lambda_i$ be the mean of the set $\lambda$. Then, we have:

\begin{eqnarray}
\mathrm{Var}(\lambda) = \frac{1}{n}\sum_{i=1}^n \left(\lambda_i-\bar{\lambda}\right)^2 
\end{eqnarray}
For all values of $i$, $\left(\lambda_i-\bar{\lambda}\right)^2 \ge 0$. Thus, the lower bound for Var($\lambda$) is 0.

\par To prove the upper bound of Var($\lambda$), observe that as $\lambda_i \le 1$ it can be said that:
\begin{eqnarray}
\sum_{i=1}^n \lambda_i^2 \le \sum_{i=1}^n \lambda_i 
\end{eqnarray}
As $n\bar{\lambda} = \sum_{i=1}^n \lambda_i $:
\begin{eqnarray}\label{eq:sum_lambdai_sq}
\sum_{i=1}^n \lambda_i^2 \le n\bar{\lambda}     
\end{eqnarray}
And so:
\begin{eqnarray}
n \cdot \mathrm{Var}(\lambda) &= \sum^n_{i=1} \left(\lambda_i-\bar{\lambda}\right)^2  \\[0.15cm]
&= \sum^n_{i=1} \left(\lambda^2_i - 2\lambda_i\bar{\lambda} +\bar{\lambda}^2\right)   \\[0.15cm]
&= \sum^n_{i=1} \lambda^2_i - 2\bar{\lambda} \sum^n_{i=1} \lambda_i + n\bar{\lambda}^2\\[0.15cm]
&= \sum^n_{i=1} \lambda^2_i - 2\bar{\lambda} \cdot n \bar{\lambda} + n\bar{\lambda}^2\\[0.15cm]
&= \sum^n_{i=1} \lambda^2_i - n\bar{\lambda}^2
\end{eqnarray}
Thus, as per Equation~\ref{eq:sum_lambdai_sq}:
\begin{eqnarray}
n \cdot \mathrm{Var}(\lambda)  \leq n\bar{\lambda} - n\bar{\lambda}^2
\end{eqnarray}
Thus: 
\begin{eqnarray}
\mathrm{Var}(\lambda)  \leq \bar{\lambda}-\bar{\lambda}^2
\end{eqnarray}
The maximum value of $\bar{\lambda} - \bar{\lambda}^2$ occurs when $\bar{\lambda} = 0.5$  and so
\begin{eqnarray}
\mathrm{Var}(\lambda) &\leq 0.5 - 0.5^2\\[0.15cm]
&\leq 0.25
\end{eqnarray}
Thus, the variance $\lambda$ is bounded [0, 0.25].

\end{proof}

\begin{theorem}\label{thm:invarianceProof}
The ratio of minimal and maximal values Z($\mathbf{a}$) for any unit vector $\mathbf{a}$:
\begin{eqnarray}\label{eq:I_variant_ratio}
\frac{\min_{|\mathbf{a}|=1}Z(\mathbf{a})}{\max_{|\mathbf{a}|=1}Z(\mathbf{a})}
\end{eqnarray}
is not invariant to uniform scaling or linear isometries.
\end{theorem}
\begin{proof}[Proof of \ref{thm:invarianceProof}]
Consider a cluster of points, $\mathcal{C}$, let $\mathcal{C}'=\alpha \mathcal{C} +\beta$ where $\alpha$ and $\beta$ are a scalar and a translation vector, respectively. Were Equation~\ref{eq:I_variant_ratio} invariant under uniform scaling and linear isomoteries, then the ratio for $\mathcal{C}'$ would be the same as that of $\mathcal{C}$. The value of $Z(\mathbf{a})$ for $\mathcal{C}'$ would be:
    \begin{eqnarray}
    \sum_{\mathbf{d} \in \mathcal{C}}\exp\left(\mathbf{a}^\intercal (\alpha \mathbf{d}+\mathbf{\beta})\right)\end{eqnarray}
Neither $\alpha$ nor $\beta$ can be factored out. Thus, the value of Equation~\ref{eq:I_variant_ratio} would change. Therefore, Equation~\ref{eq:I_variant_ratio} is not invariant under uniform scaling or linear isometries.
\end{proof}

\begin{theorem}\label{thm:upperBoundProof}
Any $I_{c|\mathcal{B}}$ is an upper bound for $I_{c,\mathrm{true}}$:
\begin{eqnarray}\forall_{c,\mathcal{B}}: I_{c,\mathrm{true}}\le I_{c|\mathcal{B}}\end{eqnarray}
\end{theorem}
\begin{proof}[Proof of \ref{thm:upperBoundProof}]
Consider the formula for $I_{c,\mathrm{true}}$:
\begin{eqnarray}
    I_{c,\mathrm{true}} = \frac{\min_{|\mathbf{a}|=1}Z'(\mathbf{a})}{\max_{|\mathbf{a}|=1}Z'(\mathbf{a})}
\end{eqnarray}
and the formula for $I_{c|\mathcal{B}}$:
\begin{eqnarray}
        I_{c|\mathcal{B}}(\mathcal{C})\approx \frac{\min_{\mathbf{b}\in \mathcal{B}}  Z'(\mathbf{b})}{\max_{\mathbf{b}\in \mathcal{B}}  Z'(\mathbf{b})}
\end{eqnarray}
where $\mathcal{B}$ is a set of unit vectors. $\mathcal{B}$ is a subset of the set of all unit vectors, $|\mathbf{a}|=1$. For it to be the case that $I_{c|\mathcal{B}} > I_{c,\mathrm{true}}$, one or both of the following must be true:
\begin{eqnarray}\min_{|\mathbf{a}|=1}Z'(\mathbf{a})>\min_{\mathbf{b}\in \mathcal{B}}Z'(\mathbf{b})\end{eqnarray}
and/or:
\begin{eqnarray}
\max_{|\mathbf{a}|=1}Z'(\mathbf{a})<\max_{\mathbf{b}\in \mathcal{B}}Z'(\mathbf{b})
\end{eqnarray}
However, as $\forall_{\mathbf{b}\in \mathcal{B}}: |\mathbf{b}| = 1$ neither of these statements can be true. Thus \begin{eqnarray}\forall_{c,\mathcal{B}}: I_{c,\mathrm{true}}\le I_{c|\mathcal{B}}\end{eqnarray}
\end{proof}

\begin{theorem}\label{thm:I_c_rnd_lim}
As the size of the set of random unit vectors used approaches infinity, $I_{c,\mathrm{rnd}}$ approaches $I_{c,\mathrm{true}}$:
\begin{eqnarray}\lim_{|r|\to \infty} I_{c,\mathrm{rnd}} = I_{c,\mathrm{true}}\end{eqnarray}

\end{theorem}
\begin{proof}[Proof of \ref{thm:I_c_rnd_lim}]
As the number of unit vectors in $r$ approaches infinity, the probability that a random unit vector is in $r$ approaches 1.

Let $\mathbf{b}$  be a unit vector. The probability of choosing $\mathbf{b}$ from the uniform distribution of the set of all unit vectors is:

\begin{eqnarray}P(\mathbf{b}) = \frac{1}{|(|\mathbf{a}| = 1)|}\end{eqnarray}

where $|(|\mathbf{a}|=1)|$ is the cardinality of the set of all unit vectors. If $r$ is sampled uniformly from the set $|\mathbf{a}| = 1$ then:
\begin{eqnarray}P(\mathbf{b}\in r) = \frac{|r|}{|(|\mathbf{a}| = 1)|}
\end{eqnarray}
Therefore:
\begin{eqnarray}\lim_{|r|\to \infty} P(\mathbf{b}\in r)  \rightarrow 1 \end{eqnarray}
and 
\begin{eqnarray}\lim_{|r|\to \infty} r \rightarrow |\mathbf{a}| = 1 \end{eqnarray}
Consequently:
\begin{eqnarray}\lim_{|r|\to \infty} I_{c,\mathrm{rnd}} = I_{c,\mathrm{true}}\end{eqnarray}
\end{proof}

\end{document}